\documentclass[letter, 10pt]{article}
\usepackage[T1]{fontenc}
\usepackage{tgpagella}
\usepackage{graphicx,amsmath,amsfonts,amscd,amssymb,bm,url,color,latexsym,amsthm}
\usepackage{fullpage}
\usepackage[small,bf]{caption}
\usepackage{subcaption}
\usepackage{bbm}
\usepackage{microtype}
\usepackage{algorithm,algorithmicx,algpseudocode} 
\usepackage{hyperref}
\usepackage{verbatim}
\usepackage{framed}
\usepackage{sidecap,wrapfig}
\usepackage[nameinlink]{cleveref}

\allowdisplaybreaks  

\hypersetup{
    colorlinks=true,%
    citecolor=blue,%
    filecolor=blue,%
    linkcolor=blue,%
    urlcolor=blue
}
\usepackage[toc, page]{appendix}

\newtheorem{theorem}{Theorem}[section]
\newtheorem{lemma}[theorem]{Lemma}

\usepackage{tikz}
\usepackage{fge}
\usepackage{pgfplots}

\usepackage{mathdots} 

\renewcommand{\mathbf}{\boldsymbol}

\newcommand{\mb}{\mathbf}
\newcommand{\mc}{\mathcal}

\newcommand{\bb}{\mathbb}
\newcommand{\set}[1]{\left\{ #1 \right\}}

\newcommand{\eps}{\varepsilon}
\newcommand{\R}{\bb R}

\newcommand{ \Brac }[1]{\left\lbrace #1 \right\rbrace}
\newcommand{ \brac }[1]{\left[ #1 \right]}
\newcommand{ \paren }[1]{ \left( #1 \right) }

\DeclareMathOperator{\sign}{sign}
\DeclareMathOperator{\diag}{diag}

\DeclareMathOperator{\st}{s.t.}

\newcommand{\norm}[2]{\left\| #1 \right\|_{#2}}
\newcommand{\abs}[1]{\left| #1 \right|}
\newcommand{\innerprod}[2]{\left\langle #1,  #2 \right\rangle}

\newcommand{\expect}[1]{\bb E\left[ #1 \right]}

\newcommand{\phihat}{\widehat{\varphi}}

\newcommand{\cconv}{\circledast}
\newcommand{\convmtx}[1]{\mb C_{#1}}
\newcommand{\extend}[1]{\widetilde{#1}}
\newcommand{\shift}[2]{s_{#2}\left[#1\right]}
\newcommand{\injector}{\mb \iota}
\newcommand{\soft}[2]{\mathrm{SOFT}_{#2}\left[ #1 \right] }
\newcommand{\revmtx}{\check{\mb C}_{\mb a_0}}
\newcommand{\grad}[2]{\mathrm{grad}\left[#1\right] \left( #2 \right )}
\newcommand{\hess}[2]{\mathrm{Hess}\left[#1\right] \left( #2 \right)}
\newcommand{\relint}[1]{\mathrm{relint}\left( #1 \right)}
\newcommand{\rbdy}[1]{\mathrm{relbdy}\left( #1 \right)}
\newcommand{\cl}[1]{\mathrm{cl}\left( #1 \right)}
\newcommand{\supp}[1]{\mathrm{supp}\left( #1 \right)}
\newcommand{\range}[1]{\mathrm{range}\left( #1 \right)}
\newcommand{\proj}[2]{\mathcal{P}_{#2}\left[ #1 \right]}

\numberwithin{equation}{section}

\def \endprf{\hfill {\vrule height6pt width6pt depth0pt}\medskip}
\renewenvironment{proof}{\noindent {\bf Proof} }{\endprf\par}

\begin{document}

\title{On the Global Geometry of\\
Sphere-Constrained Sparse Blind Deconvolution}
\author{Yuqian Zhang$^{1}$, Yenson Lau$^{2}$, Han-Wen Kuo$^{2}$,\\
        Sky Cheung$^{3}$, Abhay Pasupathy$^{3}$, John Wright$^{2,4}$\\ 
\small{ $^1$Department of Computer Science, Cornell University }\\
\small{ $^2$Department of Electrical Engineering and Data Science Institute, Columbia University} \\
\small{ $^3$Department of Physics, Columbia University }\\
\small{ $^4$Department of Applied Physics and Applied Mathematics, Columbia University}
}
\maketitle

\begin{abstract}
Blind deconvolution is the problem of recovering a convolutional kernel $\mb a_0$ and an activation signal $\mb x_0$ from their convolution $\mb y = \mb a_0 \circledast \mb x_0$. This problem is ill-posed without further constraints or priors. This paper studies the situation where the nonzero entries in the activation signal are sparsely and randomly populated. 
We normalize the convolution kernel to have unit Frobenius norm and cast the sparse blind deconvolution problem as a nonconvex optimization problem over the sphere. With this spherical constraint, every spurious local minimum turns out to be close to some signed shift truncation of the ground truth, under certain hypotheses. This benign property motivates an effective two stage algorithm that recovers the ground truth from the partial information offered by a suboptimal local minimum. This geometry-inspired algorithm recovers the ground truth for certain microscopy problems, also exhibits promising performance in the more challenging image deblurring problem. Our insights into the global geometry and the two stage algorithm extend to the convolutional dictionary learning problem, where a superposition of multiple convolution signals is observed.
\end{abstract}

\section{Introduction}
Blind deconvolution aims to recover two unknown signals: a kernel $\mb a_0$ and some underlying signal $\mb x_0$ from their convolution $\mb y = \mb a_0 \circledast \mb x_0$. Blind deconvolution is ill-posed in general: there are infinitely many pairs of signals rendering the same convolution. To render the problem well-posed, one may exploit prior knowledge about the structure of $\mb a_0$ and $\mb x_0$. For example, the underlying signal $\mb x_0$ is {\em sparse} in many engineering and scientific applications:

{\bf \em Microscopy data analysis}: In the crystal lattice of nanoscale materials, there exist randomly and sparsely distributed ``defects'', whose locations and signatures encode crucial information about the electronic structure of the material. Accurate recovery of such information can facilitate investigation of the detailed structure of materials \cite{Cheung17-Nature}.  

{\bf \em Neural spike sorting}: Neurons communicate by firing brief voltage spikes, whose characteristics reflect important features of the neuron. These spikes occur randomly and sparsely in time. Neurophysiologists are interested in assigning stereotyped spikes to putative cells, as well as in knowing their respective spike times \cite{Ekanadham2011-NIPS, Lewicki1998}.

{\bf \em Image deblurring}: Motion blur can be modeled as the convolution of a latent sharp image and a kernel capturing the motion of the camera, usually assumed to be invariant across the image \cite{Fergus2006-ACM}. The inverse process of recovering the original sharp image from a blurry image has been widely studied \cite{Cho2009-SIGGRAPH, Kundur1996-SPM}. Many well-performing approaches leverage the observation that sharp natural images typically have (approximately) sparse gradients \cite{Chan1988-TIP, Levin2011-PAMI, Perrone2014-CVPR}.
\vspace{.05in}

All of these applications lead to instances of the sparse blind deconvolution problem. The dominant algorithmic approach to sparse blind deconvolution involves {\em nonconvex optimization}\footnote{In signal processing, a number of elegant {\em convex} relaxations of the problem have been developed \cite{Ahmed2012-BDconvexprog, Chi2016-TIP, Li16-pp}. However, these approaches typically require stronger prior information (subspace constraint rather than sparsity) or exhibit suboptimal scalings.}.
Nonconvex formulations for deconvolution can be derived via several probabilistic formalisms (ML/MAP, VB, ect.), or simply from heuristics. For example, in image deblurring,  the kernel $\mb a$ can be modeled as residing on a simplex \cite{Levin2011-PAMI, Gong2016-CVPR, Krishnan2011-CVPR, Liu2014-TIP}. This is natural from a modeling prospective\footnote{Since entries of $\mb a$ roughly represent the fraction of the camera exposure time at a given location.}, but problematic for optimization: natural formulations of deconvolution over the simplex admit trivial global minimizers (corresponding to spiky convolutional kernels $\mb a=\mb\delta$) \cite{Perrone2014-CVPR, Benichoux2013-ICASSP}, which provide no information about the ground truth. Practical remedies for this problem include exploiting additional data priors \cite{Gong2016-CVPR, Liu2014-TIP, Xiao2016-ECCV} or careful initialization via edge restoration or multi-scale refinement \cite{Krishnan2011-CVPR, Xu2010-ECCV}, to avoid the trivial spiky global minima. 

In contrast, motivated by a careful comparison of MAP and VB approaches, \cite{Wipf2013-BayesianBD, Zhang2013-CVPR} propose to instead constrain $\mb a$ to have unit Frobenius norm -- i.e., to reside on a high-dimensional sphere.\footnote{\cite{Wipf2013-BayesianBD} contains a wealth of additional ideas about the role of sparsity-promoting priors in obtaining good local minima, and on the probabilistic underpinnings of deconvolution. Our experiments support the viewpoint that the key insight in \cite{Wipf2013-BayesianBD} is the role of the spherical constraint in avoiding bad minimizers.} 
This choice is arguably more appropriate for certain scientific applications -- such as microscopy -- in which the kernel $\mb a$ can have negative entries. For image deblurring, $\mb a$ can be assumed to be nonnegative, and the sphere constraint seems less natural from a modeling perspective. 
 
In this paper, we study the geometry of sphere-constrained sparse blind deconvolution. Our goal is to understand when simple algorithms based on nonconvex optimization can exactly recover the convolutional kernel $\mb a$ and the sparse signal $\mb x$. This goal is motivated by the applications described above -- in particular, microscopy data analysis -- in which there is a strong, physical sparsity prior and a clear, physical notion of the ground truth. We develop our theory and algorithms under the assumption that $\mb a$ is a short kernel, and that $\mb x$ is sparsely and randomly supported. We demonstrate through a theoretical analysis of certain (idealized) cases and many numerical experiments that when these assumptions are satisfied, the proposed algorithm correctly recovers $\mb a$, and hence $\mb x$. These results stem from a striking geometric property of sphere-constrained sparse blind deconvolution: although the problem is still nonconvex, every local minimizer $\mb{\bar a}$ is very close to a signed shift-truncation of the ground truth kernel $\mb a_0$. This observation provides a geometric explanation of how the sphere constraint can facilitate sparse blind deconvolution. 

The remainder of this paper is organized as follows. Section 2 discusses the intrinsic symmetries associated with the convolutional operator and their implication on the geometry of sphere-constrained sparse blind deconvolution. Section 3 introduces the optimization-based two stage algorithm and some related technical details. Section 4 discusses two other important extensions in image deblurring and convolutional dictionary learning. Section 5 gives experimental corroboration of our theory, and shows promising results on microscopy data analysis, image deblurring, and convolutional dictionary learning. Section 6 discusses directions for future work.

For simplicity, we assume that the convolutional signals are one dimensional in both our problem formulation and technical proofs; all of our results extend naturally to higher-dimensional signals. Throughout this paper, vectors $\mb v\in\mathbb{R}^k$ are indexed as $\mb v =[v_0,v_1,\cdots,v_{k-1}]$, and $[\cdot]_{m}$ denotes the modulo-$m$ operation. We use $\norm{\cdot}{}$ to denote the operator norm, and $\norm{\cdot}{p}$ to denote the entry wise $\ell^p$ norm. A projection onto the Frobenius sphere is denoted with $\proj{\cdot}{\bb S}=\frac{\cdot}{\norm{\cdot}{2}}$, and a projection onto subset $I$ is denoted with $(\cdot)_{I}$.

\section{Symmetry and Global Geometry}
Without loss of generality, we assume the observation data $\mb y$ is generated via a circular convolution $\cconv$ of the ground truth $\mb a_0\in\R^k$ and $\mb x_0\in\R^m$:
\begin{equation}
\mb y(\mb a_0,\mb x_0)=\mb a_0\cconv\mb x_0=\extend{\mb a_0}\cconv\mb x_0\in\R^m.
\label{eqn:circulant_conv}
\end{equation}
Here, $\extend{\mb a_0}\in\R^m$ denote the zero padded $m$-length version of $\mb a_0$, which can be expressed as $\extend{\mb a_0}=\injector\mb a_0$ with $\injector:\R^k\to\R^m$ be a zero padding operator.
Its adjoint $\injector^*:\R^m\to\R^k$ acts as a projection to lower dimensional space by keeping the first $k$ components. Equivalently, we can write
\begin{equation}
\mb y(\mb a_0,\mb x_0)=\mb C_{\extend{\mb a_0}}\mb x_0=\mb C_{\mb x_0}\extend{\mb a_0}.
\end{equation}
Here, $\mb C_{\mb v}\in\R^{m\times m}$ is the circulant matrix generated from vector $\mb v$, whose $j$-th column is a cyclic shift $s_{j-1}[\mb v]$ of the vector $\mb v$: 
\begin{equation}
\shift{\mb v}{\tau}(i)=\mb v([i-\tau]_m),\quad\forall\;i\in[0,\cdots,m-1].
\end{equation}

\subsection{Symmetries and Symmetry Breaking}
The SBD problem exhibits a {\em scaled-shift symmetry}, which derives from the symmetries of the convolution operator. Namely, given a pair $(\mb a_0, \mb x_0)$ satisfying $\mb y = \mb a_0 \circledast \mb x_0$, for any nonzero scalar $\alpha$ and integer $\tau$ 
\begin{equation}
\mb y = \left( \alpha s_{\tau}[ \widetilde{\mb a_0} ] \right) \circledast \left( \alpha^{-1} s_{-\tau} [ \mb x_0 ] \right).
\end{equation} 
Note that a scaled shift $\alpha^{-1} s_{-\tau}[ \mb x_0 ]$ of a sparse signal $\mb x_0$ remains sparse, and that a scaled shift $\alpha s_{\tau}[ \widetilde{\mb a_0} ]$ of a length-$k$ kernel $\mb a_0$ still has $k$-nonzero entries. So, these symmetries are intrinsic to the SBD problem, as formulated here. We can only hope to recover $(\mb a_0, \mb x_0)$ up to this symmetry. 

The presence of nontrivial symmetries is a hallmark of bilinear problems arising in practice -- see, e.g., \cite{SQW15-pp, SQW16-pp} for examples from dictionary learning and generalized phase retrieval. Symmetries render straightforward approaches to convexify the problem ineffective.\footnote{Given any set of points where the convex objective function achieves equal values, the function value will be no larger at any convex combination of them.} 
They also raise challenges for nonconvex optimization: equivalent symmetric solutions correspond to multiple disconnected global optima. This creates a very complicated objective landscape, which could potentially also contain spurious local optimizers.  Certain highly symmetric nonconvex problems arising in signal processing do not exhibit spurious minimizers \cite{SQW15-pp, SQW16-pp}, however, proving this can be challenging. 

\subsubsection{Symmetry breaking.}We employ a weak symmetry breaking mechanism by constraining $\mb a \in \mathbb{S}^{k-1}$ \footnote{This is motivated in part by \cite{SQW15-pp}, which demonstrates that a certain formulation of the dictionary learning problem over the sphere has no spurious local minimizers, even for relatively dense target representations. The ``simplex constrained'' analogue of that work, which optimizes over hyperplanes, requires the target solution to be much sparser \cite{Spielman12-pp}.}: we reduce the scale ambiguity to a sign ambiguity, by constraining $\mb a$ to have unit Frobenius norm; we mitigate the shift ambiguity by constraining $\mb a$ to be supported on the first $k$ entries. 

In general, $s_{\tau}[ \widetilde{\mb a_0}]$ are not be supported on the first $k$ entries, hence constraining $\mb a$ to be supported on the first $k$ entries removes the shift symmetry. However, effects of such shift symmetry still persist. Since the restriction $\injector^* s_{\tau}[\widetilde{\mb a_0}]$ to the first $k$ entries can be convolved with the sparse signal $s_{-\tau}[\mb x_0]$ to produce a near approximation to $\mb y$:
\begin{equation}
\left( \injector^* s_{\tau}[ \widetilde{\mb a_0} ] \right) \circledast s_{-\tau}[ \mb x_0 ] \approx \mb y,
\end{equation} 
especially when the shift $\abs{\tau}$ is small. We will see that (i) these symmetric solutions $\injector^* s_{\tau}[\widetilde{\mb a_0}]$ persist as local minima of a natural optimization formulation of the SBD problem, but that (ii) under conditions, these are the {\em only} local minima.

\subsection{Global Geometry on the Sphere}
We study the following objective function, which can be viewed as balancing sparsity of $\mb x$ with fidelity to the observation $\mb y$: \footnote{Similar formulation can be found in lot of sparse representation problems \cite{Mairal2014}.}
\begin{equation}
\label{eqn:obj}
\min_{\mb a\in\bb S^{k-1},\mb x}\psi(\mb a,\mb x)\doteq\tfrac{1}{2}\|\mb y-\mb a\cconv\mb x\|_2^2+\lambda r(\mb x).
\end{equation}
When $\mb x_0$ is long and random, it is more convenient to study this function through its ``marginalization'' 
\begin{equation}
\label{eqn:phi_def}
\varphi(\mb a)\doteq\min_{\mb x}\psi(\mb a,\mb x),
\end{equation}
which is defined over the sphere $\mathbb{S}^{k-1}$.

In Figure \ref{fig:geometry}, we plot the function value of $\varphi(\mb a)$ on the sphere $\mb a\in\bb S^2$: red and blue imply larger and smaller objective value respectively and there are several local minima. For this highly nonconvex function, the ground truth $\mb a_0$ achieves the global minimum, while other local minima $\mb{\bar a}$ are very close to certain signed shift truncations of the ground truth. Figure \ref{fig:geometry} (right) exhibits an example of a local minimum in a higher-dimensional problem. 

\begin{figure*}
\begin{center}
\includegraphics[width=.3\textwidth]{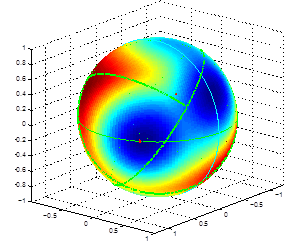}
\hspace{.2in}
\includegraphics[width=.4\textwidth]{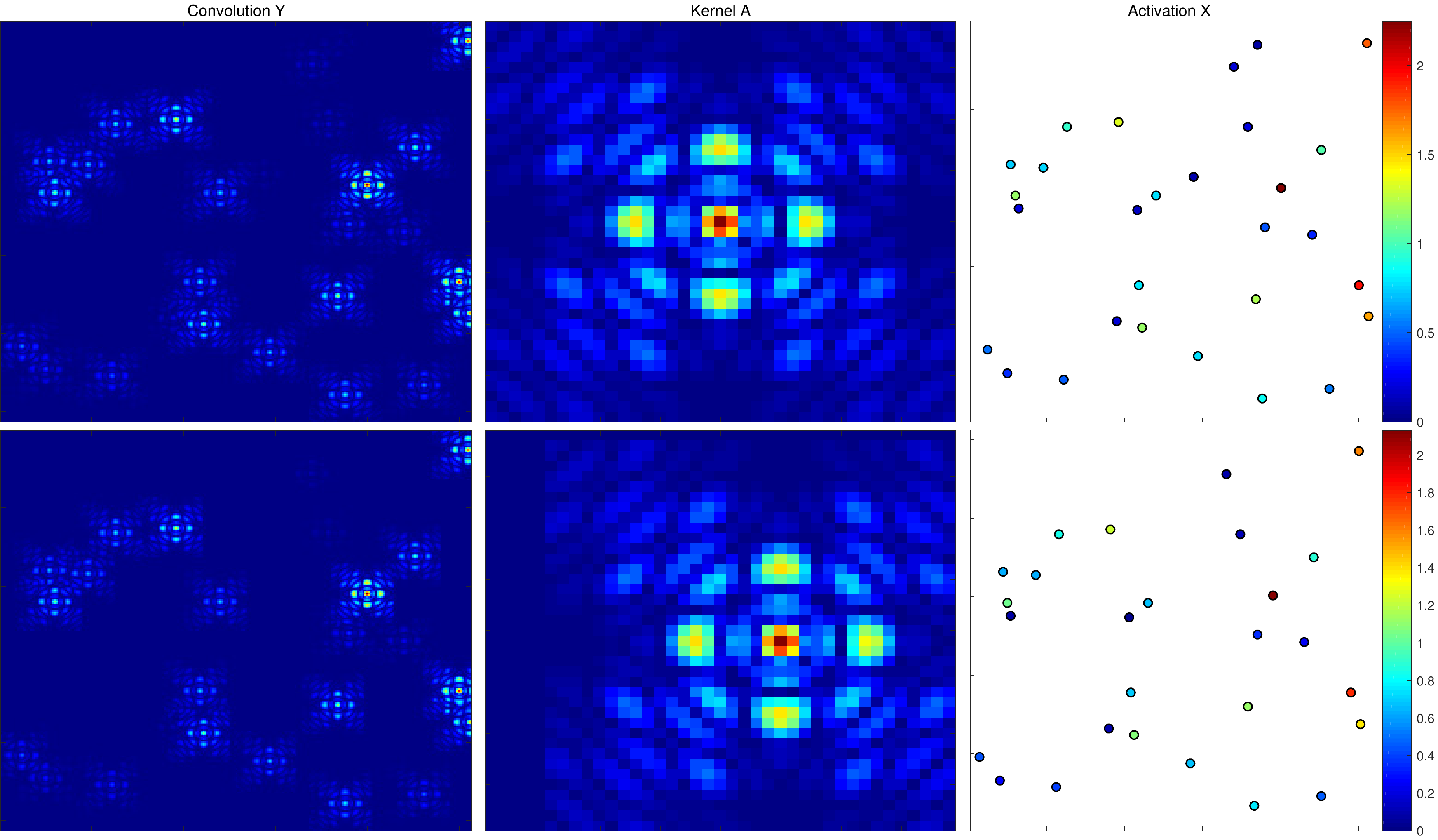}
\end{center}
\vspace{-.2in}
\caption{{\bf Geometry on the $\ell^2$ ball} for fixed $\mb{a}_0$ and generic $\mb{x}_0$. Left: the objective $\varphi(\mb{a})$ in a low dimensional setting $\mb{a}\in\bb{S}^2$ -- dark blue represents small values while dark red represents large values. All local minima are close to signed shift truncations of the ground truth $\mb{a}_0$, with $\mb{a}_0$ itself achieving global minimum. The green lines indicate regions where $\mb{a}$ are ill-posed as convolutional kernels.
Right: a shift truncation $\mb{a}$ achieves a local minimum of $\varphi(\mb{a})$ in a high dimensional setting. Shown here is the ground truth $\mb{y}=\mb{a}_0\cconv\mb{x}_0, \mb{a}_0$, and $\mb{x}_0$ (top right) versus their respective recovered quantities $\mb{a}\cconv\mb{x}, \mb{a}$, and $\mb{x}$ (bottom right).
}
\label{fig:geometry}
\end{figure*}

\subsubsection{Analysis under Restricted Settings}

Demonstrating that this observation holds in general is challenging: for most reasonable choices of the regularizer $r$, there is no closed form expression for the objective $\varphi( \mb a )$. We develop an analysis under several simplifying assumptions. Throughout, we let $r(\mb x)$ be the $\ell^1$ norm, although similar conclusions hold for other sparsifying regularizers. With this choice, we can simplify the objective $\varphi$ by dividing the sphere $\mathbb{S}^{k-1}$ via the sign-support pattern of the minimizing $\mb x^*(\mb a)$:
\begin{equation}
\mb x^*(\mb a)=\arg\min_{\mb x}\tfrac12\norm{\mb y-\mb a\cconv\mb x}2^2+\lambda\norm{\mb x}1.
\label{eqn:lasso}
\end{equation}
Let $I$ and $\mb\sigma$ denote the support and sign of $\mb x^*$
\begin{equation}
I=\supp{\mb x^*},\quad\mb\sigma=\sign(\mb x^*),
\end{equation}
then the whole sphere can be divided via the sign support pattern
\begin{equation}
\bb S^{k-1}=\bigcup_{\mb\sigma}\mc R_{\mb\sigma},\quad \mc R_{\mb\sigma}=\Brac{\mb a\mid\sign(\mb x^*(\mb a))=\mb\sigma}.
\end{equation}
On each $\mc R_{\mb\sigma}$ where the sign support pattern $\mb\sigma$ remains the same, the stationarity condition for minimizer $\mb x^*(\mb a)$ implies
\begin{equation}
\mb x^*_{I}(\mb a)=\paren{\mb C^*_{\mb a}\mb C_{\mb a}}^{-1}_{I}\paren{\mb C^*_{\mb a}\mb y-\lambda\mb\sigma}_{I}.
\end{equation}
Plugging above expression back to the original objective function $\varphi(\mb a)$ yields
\begin{equation}
\label{eqn:phi_sigma}
\varphi_{\mb\sigma}(\mb a)=-\tfrac12\paren{\mb C^*_{\mb a}\mb y-\lambda\mb\sigma}^*_{I}\paren{\mb C^*_{\mb a}\mb C_{\mb a}}^{-1}_{I}\paren{\mb C^*_{\mb a}\mb y-\lambda\mb\sigma}_{I}+ \tfrac12\norm{\mb y}2^2.
\end{equation}
Although the objective function $\varphi(\mb a)$ can be substantially simplified by removing the $\mb x$ variable in this way, it still maintains a complicated dependence on $\mb a$. To obtain some preliminary insights, we make two simplifications for easier calculation while preserving important characteristics of the geometry of $\varphi$:

\vspace{.05in}
{\bf Simplification I: $\mb x_0=\mb\delta$.} We maximally simplify the underlying sparse signal as a single spike $\mb\delta$ and the observation will be $\mb y=\mb a_0\cconv\mb x_0=\extend{\mb a_0}$. This case itself is trivial, but its function geometry is a basic but important case to be understood. This simple case also yields intuitions that carry over to the less trivial situation in which $\mb x_0$ is a long random vector.

\vspace{.05in}
{\bf Simplification II: $\mb C^*_{\mb a}\mb C_{\mb a}\to\mb I$.} For a random $\mb a\in\bb S^{k-1}$, its expectation satisfies $\expect{\mb C^*_{\mb a}\mb C_{\mb a}}=\mb I$. Here, we simply use the identity matrix to replace any $\mb C^*_{\mb a}\mb C_{\mb a}$ and therefore reduce the complexity of Equation \ref{eqn:phi_sigma}.

With these two simplifications, the original objective problem can be replaced with the following:
\begin{equation}
\text{minimize} \quad \phihat(\mb a) \quad \text{subject to} \quad \mb a \in \bb S^{k-1},
\end{equation}
where
\begin{equation*}
\phihat(\mb a)\doteq\min_{\mb x} \; \tfrac12\norm{\extend{\mb a_0}}2^2 + \tfrac{1}{2} \norm{\mb x}{2}^2  - \innerprod{ \extend{\mb a} \cconv \mb x }{\extend{\mb a_0}} + \lambda \norm{\mb x}1.
\label{eqn:phihat}
\end{equation*}
In this case, the minimizing $\mb x^*(\mb a)$ has a simple closed form solution:
\begin{equation}
\mb x^*(\mb a)=\soft{\convmtx{\mb a}^* \extend{\mb a_0}}{\lambda} = \soft{\revmtx^* \injector \mb a}{\lambda},
\end{equation}
here, $\soft{u}{\lambda} = \sign(u) \max\set{ |u| - \lambda, 0 }$ is the entry-wise soft-thresholding operator and $\revmtx\in \R^{m \times m}$ is the reversed circulant matrix for $\mb a_0$ defined via
\begin{equation} 
\revmtx = \left[ \begin{array}{c|c|c|c} \shift{\extend{\mb a_0}}{0} & \shift{\extend{\mb a_0}}{-1} & \dots & \shift{\extend{\mb a_0}}{-(m-1)} \end{array} \right].
\end{equation}
On a constant sign support pattern $\mb\sigma$, $\phihat(\mb a)$ can be written into a simpler quadratic form: 
\begin{equation}
\phihat_{\mb\sigma}(\mb a)=-\tfrac12\norm{\paren{\mb C^*_{\mb a}{ \extend{\mb a_0}}-\lambda\mb\sigma}_{ I}}2^2+\tfrac12\norm{\extend{\mb a_0}}2^2.
\label{eqn:phihat_sigma}
\end{equation}
For this surrogate $\phihat(\mb a)$, we can show that if $\lambda$ is sufficiently large compared to the magnitude of $\mb x_0$, every strict local minimizer is a signed shift truncation of the ground truth:

\begin{theorem}
\label{thm:lm-nearly-orthogonal} 
Define the set of possible supports of minimizer $\mb x^*$ with 
\vspace{-.05in}
\begin{equation}
\mc I  = \set{ \mathrm{supp}\left( \soft{ \revmtx^* \injector \mb a }{\lambda } \right) \mid \mb a \in \bb S^{k-1} }.
\vspace{-.05in}
\end{equation}
For each nonempty support $I = \set{ i_1 < i_2 <\cdots < i_{|I|} }$, let

\begin{equation*}
\mb W_I = \brac{ \frac{\injector^* \shift{\extend{\mb a_0}}{-i_1} }{ \norm{\injector^* \shift{\extend{\mb a_0}}{-i_1} }{F} } 
\middle| \frac{\injector^* \shift{\extend{\mb a_0}}{-i_2} }{ \norm{ \injector^* \shift{\extend{\mb a_0}}{-i_2} }{F} } 
\middle| \dots \middle| \frac{\injector^* \shift{\extend{\mb a_0}}{-i_{|I|}} }{ \norm{\injector^* \shift{\extend{\mb a_0}}{-i_{|I|}}}{F} } }
\end{equation*}
Suppose that $\lambda<1$ and that for every nonempty $I \in \mc I$, 
\begin{equation}
\norm{ \mb W_I^* \mb W_I - \mb I }{} < \frac{\lambda^2 }{6},
\end{equation}
then every local minimum $\bar{\mb a}$ of $\phihat$ over $\bb S^{k-1}$ satisfies either $\bar{\mb a} \in\mc R_{\mb 0}$ (in which case $\bar{\mb a}$ is also a global maximum), or 
\begin{equation}
\bar{\mb a} = \pm \frac{ \injector^* \shift{\extend{\mb a_0}}{\tau} }{ \norm{\injector^* \shift{\extend{\mb a_0}}{\tau} }{F} }
= \pm \proj{\injector^* \shift{\extend{\mb a_0}}{-\tau} }{\bb S}
\end{equation}
with $\mb x^\star(\bar{\mb a})=\pm\soft{ \norm{\injector^* \shift{\extend{\mb a_0}}{\tau}}{F}}{
\lambda}\shift{\mb x_0}{-\tau}$ for some shift $\tau$. 
\end{theorem}

\begin{proof}
Please refer to the supplement.
\end{proof}

This theorem says that the only local minima in this idealized case are signed shift truncations of the ground truth $\mb a_0$, with certain choice of $\lambda$. Moreover, on those local minima $\bar{\mb a}$, the minimizing sparse $\mb x^\star(\bar{\mb a})$ correspond to the soft thresholded, oppositely shifted ground truth $\mb x_0$. The quantity $\norm{ \mb W_I^* \mb W_I - \mb I }{}$ measures the orthogonality of different shifts of $\mb a_0$. In particular, if $\mb a_0$ is benign enough in the sense that any two different shifts of $\mb a_0$ are uncorrelated, or $\norm{ \mb W_I^* \mb W_I - \mb I }{}\to0$ for any $I$, then any nonzero $\lambda$ guarantees the desired geometry. On the other hand, for a fixed $\mb a_0$, both $\abs{\mc I}$ and $\abs{I}$ becomes smaller as $\lambda$ increases, therefore the constraint $\norm{ \mb W_I^* \mb W_I - \mb I }{}< \lambda^2 /6$ is more likely to be satisfied. 

A similar result holds when $\mb x_0$ is separated enough that copies of the kernel do not overlap. It also holds if $\mb x_0$ is a long, sufficiently sparse random vector. For example, if the entries of $\mb x_0$ satisfy a Bernoulli-Gaussian distribution $\mb x_0(i) = \Omega(i) \mb v(i)$, with $\Omega(i) \sim \mathrm{Ber}(\theta)$ and $\mb v(i) \sim \mc N(0,1)$, and the probability $\theta$ diminishes sufficiently quickly with $k$. We conjecture that this phenomenon holds much more broadly. In particular, determining how slowly $\theta$ can diminish with $k$ is an open problem.

\subsection{Global Geometry on the Simplex}
In the application to image deblurring, the blur kernel is always positive and sums to $1$, which naturally leads to a simplex-constrained optimization problem. However, the optimization landscape changes drastically when the convolutional kernel $\mb a$ is constrained to live on the $\ell^1$ norm ball. The objective value of the same objective function over the $\ell^1$ ball is shown in Figure \ref{fig:geometry_simplex}. 
\begin{figure}[H]
\centering
\includegraphics[width=.7\textwidth]{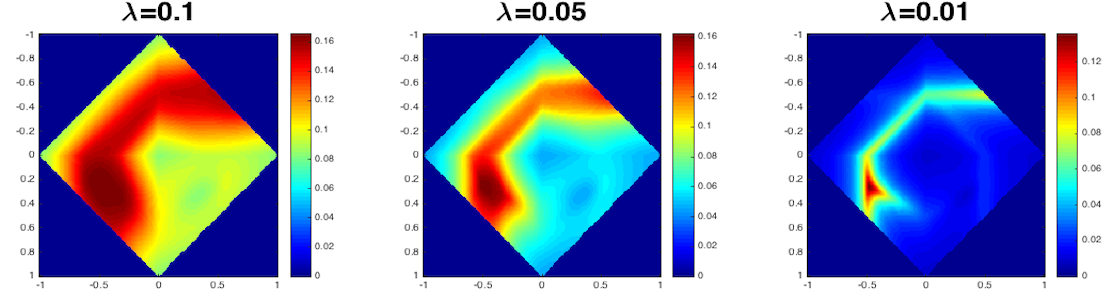}
\vspace{-.1in}
\caption{{\bf Geometry on the $\ell^1$ ball:} The trivial spike convolutional kernel is the global minimizer, while the ground truth $[1/3, 1/3, 1/3]$ becomes a local minimizer.}
\label{fig:geometry_simplex}
\end{figure} 

There is a significant difference induced by these two constraints: the trivial spike kernel ($\mb a=\mb\delta$) becomes the global minima and other meaningful solutions become local minima with the $\ell^1$ norm constraint \cite{Benichoux2013-ICASSP}, while the spherical constraint always renders local minima close to some signed shift truncation of the ground truth. This important empirical knowledge of the structure of the local minima enables us to infer the ground truth from any local minimizer.

\section{A Two-Stage Algorithm}
Inspired by the geometric property that every local minimum of the simplified problem $\phihat$ is a signed shift-truncation of the ground truth $\mb a_0$, we present a two stage algorithm for reliable recovery of the ground truth $\mb a_0$ in this section. In the first stage, the algorithm recovers some signed shift truncation of the ground truth, and the following stage infers the ground truth from this partial recovery.

\subsection{Stage I: Find the Signed Shift Truncation}
Theorem \ref{thm:lm-nearly-orthogonal} suggests that $\lambda$ needs to be relatively large to guarantee that all local minimizers are signed shift truncations of the ground truth. However, this is not sufficient to guarantee the success of an optimization algorithm due to the non-differentiability of the $\ell^1$ regularizer at $\mb x^\star=\mb 0$. Because of this non-differentiability, when $\lambda$ is too large, there is a nonzero measure set of $\mb a$ where $\norm{ \mb C^*_{\mb a} \mb y }{\infty} \le \lambda$ and therefore $\mb x^\star(\mb a) = \mb0$. These $\mb a$ are not correlated with any signed shift truncation of $\mb a_0$ and are the global maxima of $\varphi$. The objective function is constant over this region, so there is no way to escape using only local information. 

One way to cope with this flat global maxima region is to replace the nondifferentiable $\ell^1$ sparsity penalty with a differentiable one. A natural choice is the huber-$\mu$ function, which can be seen as an $\ell^1$ penalty but with a rounded bottom for $\abs{x_i}\leq \mu$:
\begin{equation}
h_{\mu}(\mb x)= \sum_{\abs{x_i}\leq \mu} \left(\frac{x_i^2}{2\mu} + \frac{\mu}{2}\right) + \sum_{\abs{x_i} > \mu} \abs{x_i}
\end{equation}
As we choose $\mu\ll\lambda$, the $\mu$-huber function closely approximates the $\ell^1$ norm, which still maintaining the effect of ``smoothing'' the flat region.
The flat region for the $\ell^1$ penalty objective occurs when $\mb x^*(\mb a) = \mb 0$, correspondingly we define a region as $\mc R_{h,0}$ with small $\mb x^*(\mb a)$ such that 
\begin{equation}
\label{eqn:R0}
\mc R_{h,0}:=\{\mb a\in\bb S^{k-1}:\norm{\mb x^*(\mb a)}{\infty} \leq \mu\}.
\end{equation}
Within the $\mc R_{h,0}$ region, the original objective function can be rewritten into a simpler form:
\begin{equation}
\varphi_{h_\mu}(\mb a)=\tfrac12\|\mb y-\mb a\cconv\mb x\|_2^2+\tfrac{\lambda}{2\mu}\|\mb x\|_2^2+\tfrac{\mu n}2,
\end{equation}
thus the optimality condition for $\mb x^*$ implies
\begin{equation}
\label{eqn:xopt_R0}
\mb x^*(\mb a)=\paren{\mb C_{\mb a}^*\mb C_{\mb a}+\tfrac{\lambda}{\mu}\mb I}^{-1}\mb C_{\mb a}^*\mb y
\;\approx\; \tfrac{\mu}{\lambda}\mb C_{\mb a}^*\mb y.
\end{equation}
Plugging $\mb x^*(\mb a)$ back to \eqref{eqn:obj} and ignoring the higher order term $O(\frac{\mu^2}{\lambda^2})$ yields
\begin{equation}
\varphi_{h_\mu}(\mb a)\approx-\tfrac{\mu}{2\lambda}\norm{\mb y\cconv\mb a}2^2+\tfrac{1}{2}\norm{\mb y}2^2+\tfrac{\mu n}2.
\end{equation}

In this case, minimization of the objective function $\varphi_{h_\mu}(\mb a)$
within the region $\mc R_{h,0}$ is equivalent to finding the maximum eigenvalue of the matrix $\mb\iota^*\mb C_{\mb y}^*\mb C_{\mb y}\mb\iota$ with the corresponding leading eigenvectors $e_1(\mb\iota^*\mb C_{\mb y}^*\mb C_{\mb y}\mb\iota)$ achieving the local minima. However, these points can be excluded from $\mc R_{h,0}$ by setting $\lambda<\min_{\mb v \in e_1(\mb \iota^*\mb C^*_{\mb y}\mb C_{\mb y}\mb \iota)}\norm{\mb C^*_{\mb y}\mb\iota\mb v}{\infty}$.\footnote{A computationally easier upper bound would be $\sqrt\frac{\lambda_1(\mb \iota^*\mb C^*_{\mb y}\mb C_{\mb y}\mb \iota)}{k}$. Note that in some scenario, there exists a local minima appearing in either region $\mc R_{h,0}$ or $\mc R_{h,0}^c$ regardless of how we set $\lambda$. Such extreme case happens when the ground truth convolutional kernel is only supported on a small consecutive portion of its full size, hence a tight estimate of the kernel size is preferred. }

With above modifications, the original flat local maxima region $\mc R_{h,0}$ becomes concave and always have a direction of negative curvature for the algorithm to escape $\mc R_{h,0}$. Hence, the first stage of the algorithm can find a signed shift-truncation of the ground truth $\bar{\mb a}  = \pm \frac{ \injector^* \shift{\extend{\mb a_0}}{\tau} }{ \norm{\injector^* \shift{\extend{\mb a_0}}{\tau} }{F}}$ as desired.

\subsection{Stage II: Infer the Ground Truth}
The second stage of the algorithm aims to recover the ground truth from its signed shift truncation $\bar{\mb a}$. To recover the truncated part, we first put $\bar{\mb a}$ in a higher dimensional sphere by zero padding (Figure \ref{fig:zero_pad}), and then recover the ground truth $\mb a_0$ on this higher dimensional sphere. Intuitively, as $\bar{\mb a}$ still captures a considerable portion of the ground truth $\mb a_0$ (the zero padded $\bar{\mb a}$ is close to the shifted $\mb a_0$ in a higher dimensional space), the zero padded $\bar{\mb a}$ serves as a good initialization. This intuition is made rigorous in the following lemma:
\begin{lemma}
\label{lem:stage2}
Let $\lambda_{rel} = \lambda/\norm{\mb x_0}{\infty}$,
suppose the ground truth $\mb a_0$ satisfies 
\begin{equation}
\abs{\innerprod{\mb a_0}{\injector\shift{\extend{\mb a_0}}{\tau\neq 0}}}<\lambda_{rel}^2-\paren{2+1/\lambda_{rel}^2}\sqrt{1-\lambda_{rel}^2}
\end{equation}
for any nonzero shift $\tau$, and $\mb x_0$ is separated enough such that any two nonzero components are at least $2k$ entries away from each other. If initialized at some $\mb a\in\bb S^{k-1}$ that $\abs{\innerprod{\mb a}{\mb a_0}}>\lambda_{rel}$, 
a small-stepping projected gradient method minimizing $\varphi(\mb a)$ recovers the signed ground truth $\pm\mb a_0$.
\end{lemma}

\begin{proof}
Please refer to the supplement.
\end{proof}
This lemma says when the initial point $\mb a$ is close enough to the ground truth, 
the gradient always points to $\mb a_0$ as long as $\abs{\innerprod{\mb a_0}{\injector\shift{\extend{\mb a_0}}{\tau\neq0}}}$ is sufficiently small. Theorem \ref{thm:lm-nearly-orthogonal} suggests that the first stage of the algorithm finds one local minimum $\mb{\bar a}$ that $\abs{\innerprod{\mb{\bar a}}{\mb a_0}}\ge \lambda_{rel}$. Hence, the second stage of the algorithm, which minimizes the same objective function but on a higher dimensional sphere, recovers the ground truth up to sign shift ambiguity as desired.

\begin{figure}[t]
\centering
\includegraphics[width=.3\textwidth]{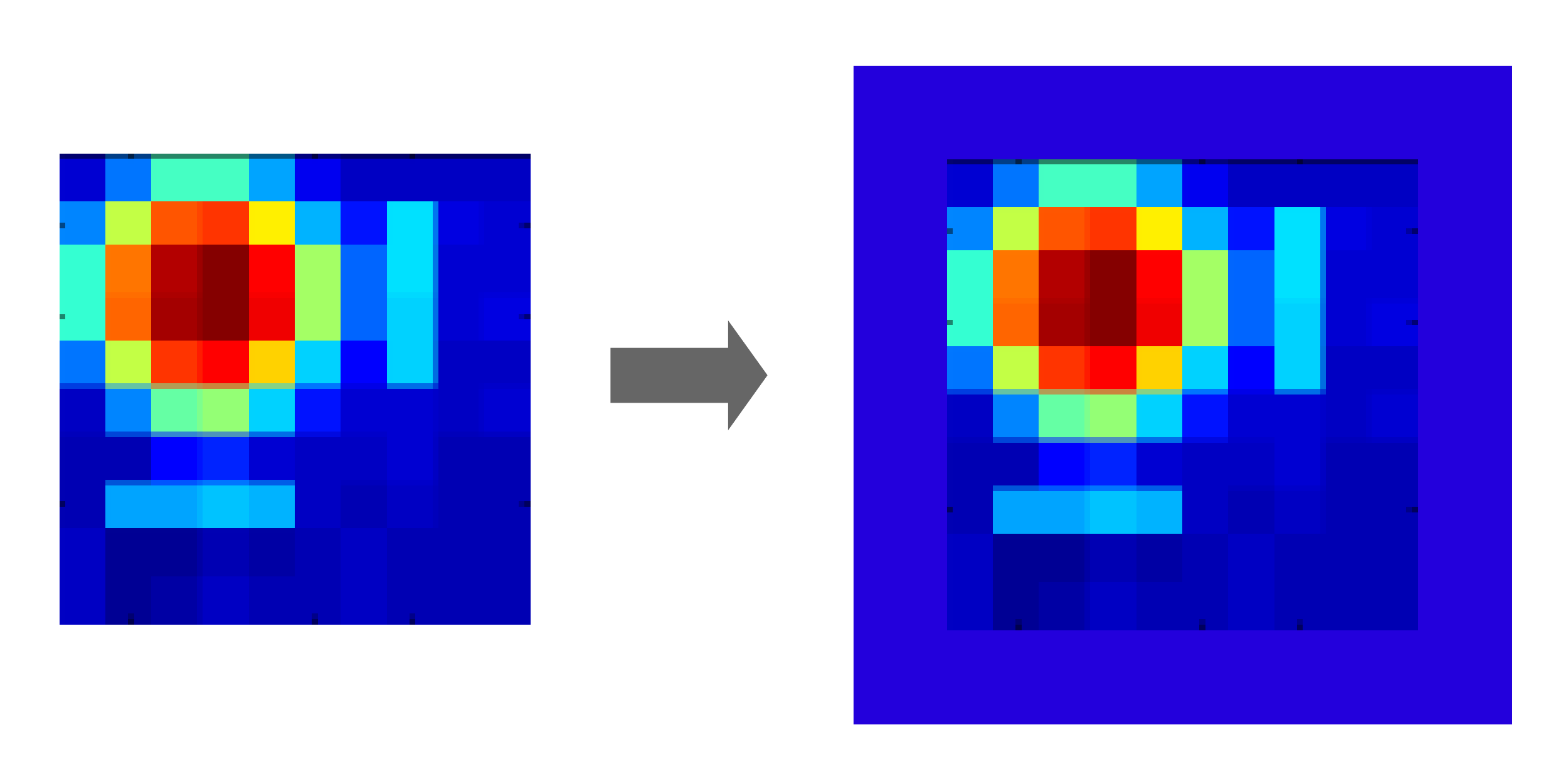}
\vspace{-.1in}
\caption{{\bf Zero Padding a Signed Shift Truncation} The original signed shift truncation (left) and the corresponding zero padded one (right).}
\label{fig:zero_pad}
\vspace{-.2in}
\end{figure}

To ensure accurate recovery, it is important to take the effect of $\lambda$ on the function geometry into consideration. A larger $\lambda$ encourages a sparser $\mb x$ and induces a simpler and smoother function landscape, which effectively eliminates undesirable local minima that are not close to any signed shift truncations, as shown in Figure \ref{fig:geo_lambda}. On the other hand, a smaller $\lambda$ emphasizes more on the accurate recovery of the signal, therefore the global minima of \eqref{eqn:obj}  will be closer to the ground truth when $\lambda$ decreases.

\begin{figure}[h!]
\centering
\includegraphics[width=.6\textwidth]{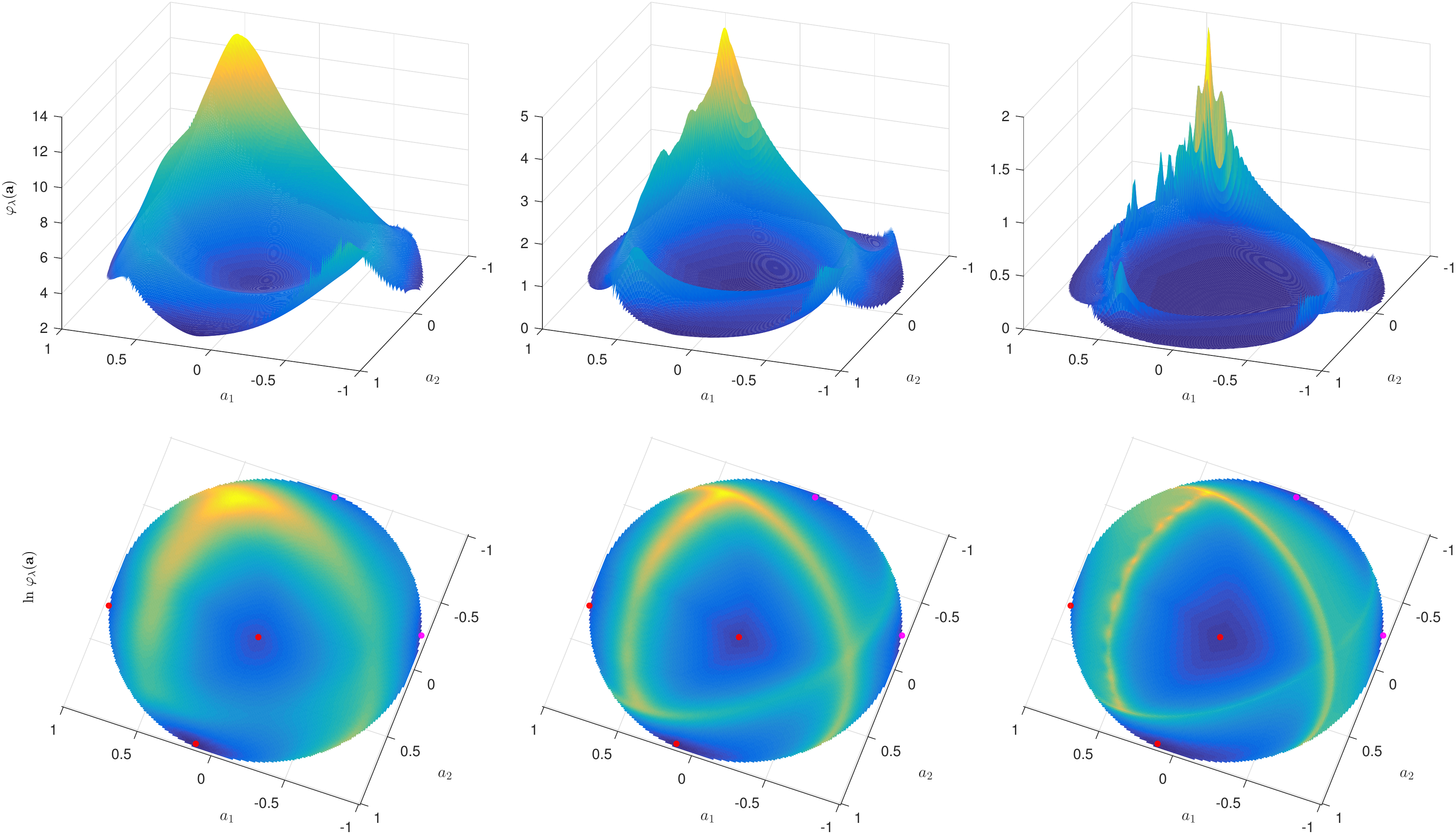}
\vspace{-.05in}
\caption{{\bf Function Geometry with Varying $\bf\lambda$:} The objective $\varphi(\mb a)$ over the hemisphere  for $\lambda=10^{-1}, 10^{-3}, 10^{-6}$. Here $\mb a_0=\proj{[1,8,2]}{\bb S^2}$ and $\mb x_0\sim \text{Ber}(.1)\odot\mc N(0,1)$. The ground truth kernel $\mb a_0$ and its shift-truncations $\proj{[8,2,0]}{\bb S^2}$, $\proj{[0,1,8]}{\bb S^2}$ are shown in red, and sign-flips $\proj{-[8,2,0]}{\bb S^2}$, $\proj{-[0,1,8]}{\bb S^2}$ are shown in magenta. Notice that each signed shift truncation shown on the hemisphere is close to a corresponding local minima, while as the objective landscape becomes less regularized as $\lambda$ shrinks.}
\label{fig:geo_lambda}
\end{figure} 

This geometric effect induced by $\lambda$ suggests a continuation method in the second stage of the algorithm. We start with a relatively big $\lambda$ for smoother function geometry, which encourages the algorithm to converge to one meaningful local minimum  close to some signed shift truncation of the ground truth. Then run the same algorithm with decreasing sequence of $\lambda$ to produce a finer approximation of the ground truth. The overall algorithm is described in Algorithm \ref{alg:nonconvexBD}.

\begin{algorithm}[h!]
\caption{Nonconvex Sparse Blind Deconvolution}
\label{alg:nonconvexBD}
\begin{algorithmic}[1]
\Ensure{Observation data $\mb y$, regularization parameter $\lambda_0$ and $\lambda_{\min}$, continuation parameter $\beta>1$}
\State Solve $\mb a^{(0)}=\arg\min\varphi_{\lambda_0}(\mb a)$ on $\bb S^{k-1}$ with random initialization\\
\State Set $\lambda_1=\lambda_0$, zero pad $\mb a^{(0)}$ to $\mb a^{(1)}$ and $\mb a^{(1)}\in\bb S^{k'-1}$ $(k'>k)$.
\While{$\lambda_k>\lambda_{min}$}
\State Solve $\mb a^{(k+1)}  = \arg\min\varphi_{\lambda_k}(\mb a)$ on $\mathbb S^{k'-1}$ with initialization $\mb a^{(k)}$. \\
\State $\lambda_{k+1} = \lambda_k/{\beta}$\\
\EndWhile
\end{algorithmic}
\end{algorithm}

We need to note that solving $\mb a=\arg\min\varphi_{\lambda}(\mb a)$ in Algorithm \ref{alg:nonconvexBD} involves iteration between (i) finding the marginalization over $\mb x^*(\mb{a})$ step, and (ii) updating $\mb a$ based on the gradient/Hessian of $\varphi_{\lambda}(\mb a)$. This could be very computationally consuming, a more efficient variant would be to optimize over the cross space of $\mb a$ and $\mb x$ together. The corresponding algorithm can be easily adapted to fit into the same general framework. The only things we want to emphasize in the proposed algorithm are the dimension lifting of the sphere and the continuation of $\lambda$.

\section{Further Extensions}

In this section, we extend our algorithm to handle two other deconvolution problems of practical interests: image deblurring and convolutional dictionary learning. The proposed two stage algorithm can be modified and applied to these more complicated applications.

\subsection{Image Deblurring}
{\em Image deblurring} aims to recover a sharp natural image from its blurred observation due to unknown photographic processes such as camera shake or defocus. Although the natural images are not necessarily sparse, it is widely acknowledged that their gradients are approximately sparse.  Let $\mb y=\mb a_0\cconv\mb x_0$ denote the observed blurry image, which is the convolution of the original sharp image $\mb x_0$ and the blurring kernel $\mb a_0$. Because of the linearity of the convolution operator, the gradient of the blurred image equals the convolution of the kernel and gradient of the original sharp image, which is usually sparse as desired
\begin{eqnarray}
\nabla_x\mb y=\mb a_0\cconv\nabla_x\mb x_0, \quad 
\nabla_y\mb y = \mb a_0\cconv\nabla_y\mb x_0.
\end{eqnarray}
Here, $\nabla_x$ and $\nabla_y$ denote derivatives in the $x$ and $y$ directions. In this application, $\nabla_x\mb x_0$ and $\nabla_y\mb x_0$ are the underlying sparse signals, and the blind image deblurring problem can be cast as solving:
\begin{eqnarray}
\min_{\mb a\in\bb S_{+}^{k-1},\mb x_1,\mb x_2}&\Bigl\{&\tfrac12\|\nabla_x\mb y-\mb a\cconv\mb x_1\|_2^2+\lambda r(\mb x_1)\\
&+&\tfrac12\|\nabla_y\mb y-\mb a\cconv\mb x_2\|_2^2+\lambda r(\mb x_2) \Bigr\}.\nonumber 
\label{eqn:bid_obj}
\end{eqnarray}
Here, $\bb S_{+}^{k-1}$ denotes the intersection of the unit sphere and the positive orthant. In this application, the non-negativity of the blurring kernel removes the sign ambiguity. We observe in experiments that local minimizers are all near some shift truncation of the ground truth kernel. The same two stage algorithm can therefore be applied to infer the ground truth.

\subsection{Convolutional Dictionary Learning}
{\em Convolutional dictionary learning} (CDL) is an important problem in machine learning for images, speeches, as well as scientific problems like microscopy data analysis and neural spike sorting. The observation signal $\mb y$ is the superposition of convolutions of $N$ pairs of kernels $\mb a_{0n}$ and corresponding coefficients $\mb x_{0n}$:
\begin{equation}
\label{eqn:multiBD}
\mb y=\textstyle{\sum_{n=1}^N}\mb a_{0n}\cconv\mb x_{0n}.
\end{equation}
Blind deconvolution can be seen as a special case of CDL with $N=1$. If the coefficients $\mb x_{0n}$ are sparse, a natural way to extend our knowledge of SBD would be to assume all $N$ convolutional kernels having unit Frobenius norm and cast it as minimizing following objective function over the product of $N$ spheres:
\begin{equation}
\label{eqn:multiBD_obj}
\min_{\mb a_n\in\bb S^{k-1}}\min_{\mb x_n}\tfrac12\|\mb y-\textstyle{\sum_{n=1}^N}\mb a_n\cconv\mb x_n\|_2^2+\lambda \textstyle{\sum_{n=1}^N}r(\mb x_n).
\end{equation}
We anticipate that all the local minima are near signed shift truncations of the ground truth, provided the target kernels $\mb a_{0n}$ are sufficiently diverse. The modified two stage algorithm  still manages to capture the partial information offered by local minima and hence recovers the ground truth. Experimental results are provided in Section \ref{subsec:CDL-exp} to corroborate this claim.

\section{Experiments}
In this section, we investigate the performance of our algorithm on both synthetic and real data. We first report a systematic investigation, performed in \cite{Cheung17-Nature}, of the performance of our algorithm on synthetic data, which are designed to mimic properties of the microscopy data analysis problem. In Sections \ref{subsec:microscopy-exp}-\ref{subsec:CDL-exp}, we present experiment results showing how our method performs on real data from microscopy and image deblurring. 
\subsection{Evaluation on Synthetic Data}
\noindent {\bf Noise-free data}: we generate the noise-free observation signal of size $m=256\times256$ through circular convolution between a kernel of size $k$ and a random underlying activation signal with a Bernoulli distribution with sparsity $\theta$, i.e. $x_i \overset{\text{i.i.d.}}{\sim} \text{Ber}(\theta)$, or $x\sim \text{Ber}(\theta)$. We plot the kernel recovery error for varying kernel size $k$ and sparsity level $\theta$ in the left of Figure \ref{fig:phase_transition} \cite{Cheung17-Nature}. Each point on the diagram is the average of 20 independent measurements. The algorithm performs excellently in the blue regions, but begins to fail in the red regions, where either the kernel size is large or the underlying activation signal is dense. The region where typical STM measurements are performed are bounded below by the white dashed line, where the proposed algorithm achieves satisfying performance. 

\noindent {\bf Noisy data}: we generate convolutional signals by convolving fixed kernel of dimension $k$, $k/m = 0.14$ with the random activation map $x\sim \text{Ber}(\theta)$ of dimension $m$, and applying additive Gaussian noise. We test the performance of our algorithm for varying sparsity $\theta$ and noise power. The result is shown in Figure \ref{fig:phase_transition} (right): the algorithm achieves noise-robust recovery when the sparsity constraint is satisfied.
\begin{figure}[h!]
\centering
\includegraphics[width=0.8\textwidth]{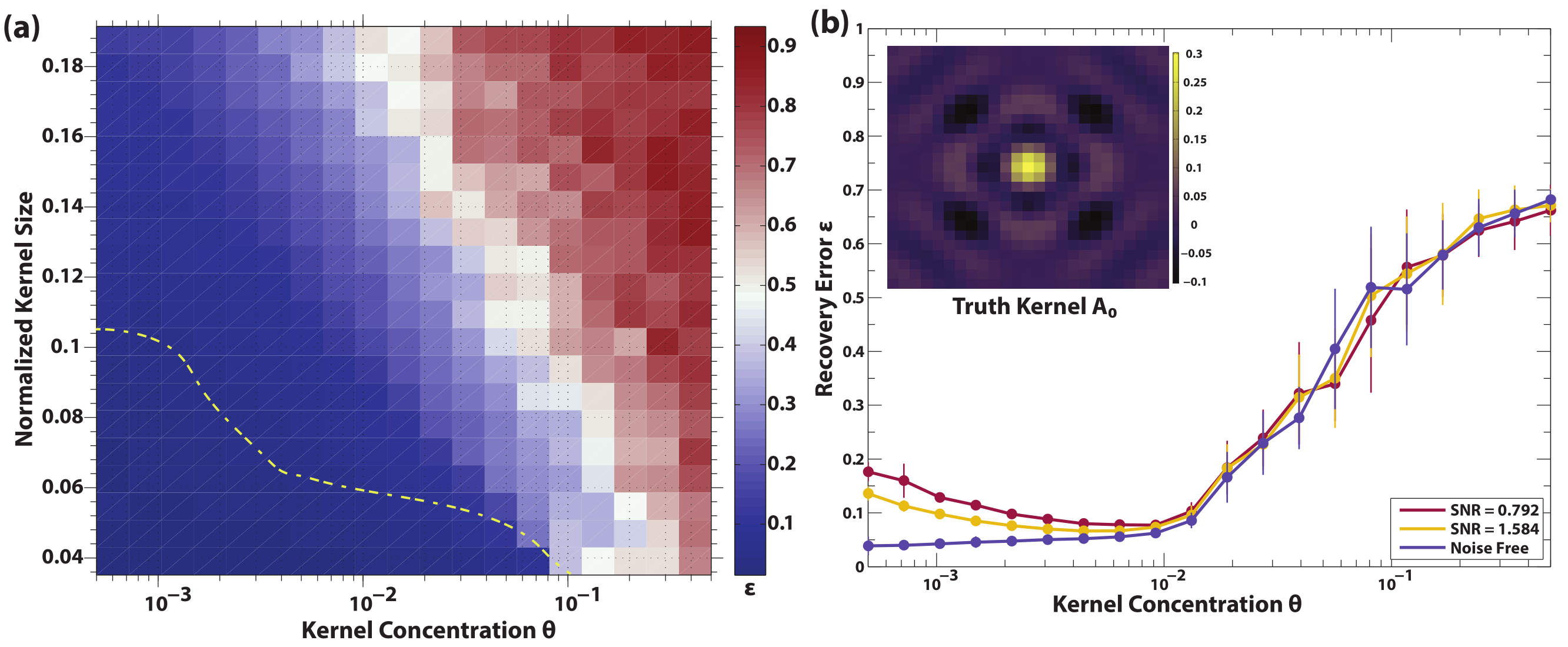}
\vspace{-.1in}
\caption{{\bf Recovery accuracy \cite{Cheung17-Nature}.} Left: phase transition diagram from noise-free simulated results. Right: performance of algorithm \ref{alg:nonconvexBD} in the presence of additive noise in the measurement; the error increases for small $\theta$ due to a lack of samples, whereas extremely large $\theta$ leads to algorithmic failure.}
\label{fig:phase_transition}
\vspace{-.15in}
\end{figure}

\subsection{Microscopy Data Analysis} \label{subsec:microscopy-exp}
We apply our algorithm on experimental microscopy data obtained from a NaFeCoAs sample. Our results shown in Figure \ref{fig:stm} indicate that the proposed algorithm manages to recover the missing details of the ripples in the Fourier domain of the defect, which encode the physical scattering processes of electrons at work.

\begin{figure}[h!]
\centering
\includegraphics[width=0.7\textwidth]{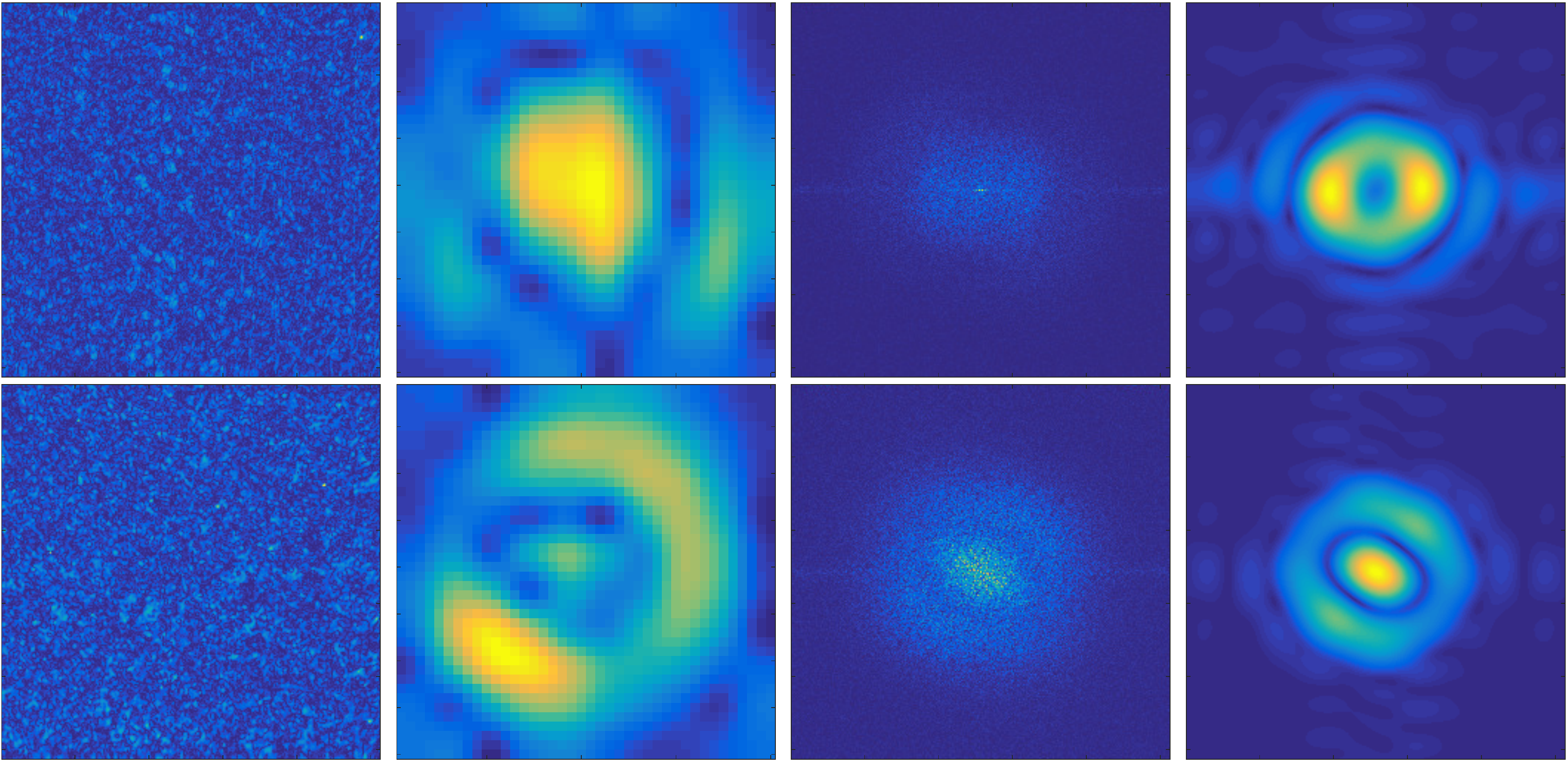}
\caption{{\bf STM Data Analysis.} From left to right: the microscopy images, extracted convolutional kernels (defect patterns), and their respective Fourier magnitude images. 
}
\label{fig:stm}
\vspace{-.2in}
\end{figure}

\subsection{Image Deblurring} 
\label{subsec:deblurring-exp}
We test our algorithm on the image deblurring dataset from \cite{Levin2011-PAMI}, solving (\ref{eqn:bid_obj}) to recover the convolutional kernel. To clearly separate the inaccuracy of the algorithm and the universal blurring kernel model, all the experiments are done on three kinds of blurred images: (i) synthetic blurred images generated by the convolution of sharp images and blurry kernels; (ii) noisy blurred images generated by adding Gaussian noise to the clean synthetic blurred images (SNR=100); and (iii) real blurry images taken with camera shakes \cite{Levin2011-PAMI}. 

We compare with algorithms by Zhang et al.\cite{Zhang2013-CVPR}, Krishnan et al.\cite{Krishnan2011-CVPR}, Sun et al.\cite{Sun2013-ICCP}, and Liu et al.\cite{Liu2014-TIP}.\footnote{We use the default parameters for these algorithms. It's possible that better performance could be obtained by tuning the parameters more carefully. In our algorithm, we fix the $\lambda$'s to be $0.1, 0.01, 0.001, 0.001$ for all the instances.} 

Because of the shift ambiguity, we evaluate the accuracy of the recovered blurring kernel considering all possible shifts. The kernel recovery error is defined as $\min_{\tau}\norm{\injector^*\shift{\extend{\mb a}}{\tau}/\norm{\mb a}{1}-\mb a_0/\norm{\mb a_0}{1}}{F}$, and the cumulative distribution is shown in Figure \ref{fig:bid_hist_kernel}.

\begin{figure}[h!]	
\centering
\hspace{-.2in}
\begin{minipage}{0.3\textwidth}
\includegraphics[width=1\textwidth]{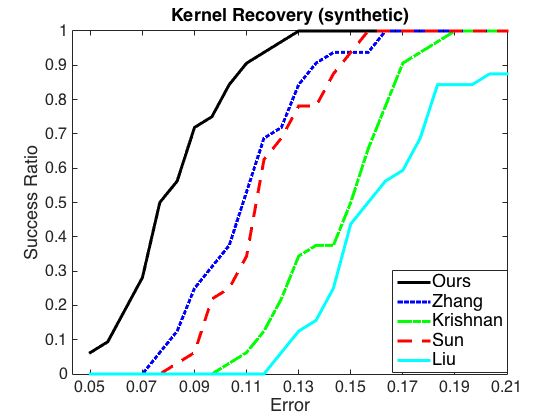}
\end{minipage}
\begin{minipage}{0.3\textwidth}
\includegraphics[width=1\textwidth]{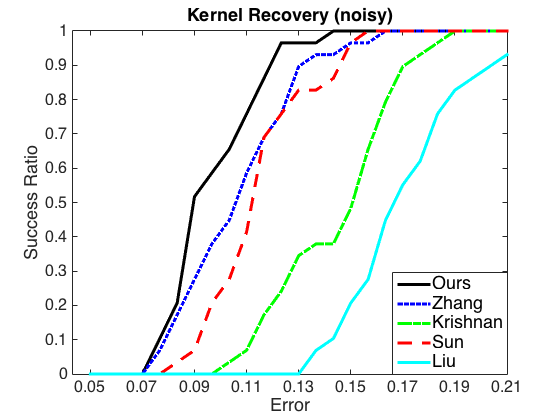}
\end{minipage}
\begin{minipage}{0.3\textwidth}
\includegraphics[width=1\textwidth]{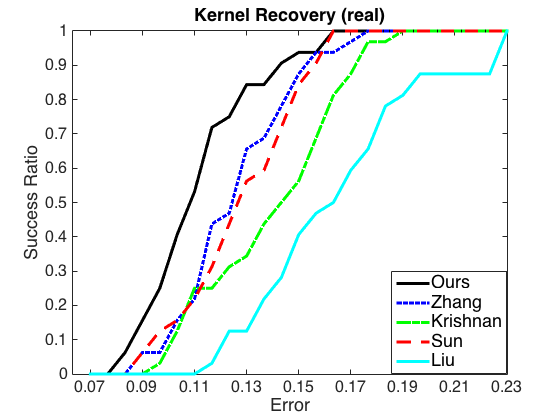}
\end{minipage}
\vspace{-.1in}
\caption{{\bf Blur Kernel Recovery Error}: Cumulative distributions of recovered blur kernel error from synthetic (left), noised(middle) and real (right) blurry images.}
\label{fig:bid_hist_kernel}
\end{figure} 

We use the same non-blind deblurring algorithm from \cite{Krishnan2009-NIPS}, with the same parameter. We consider the blurred image using the ground truth kernel to be the bench mark, and evaluate the quality of the deblurred image by calculating the Frobenius norm of its difference to such bench mark. Results are shown are in Figure \ref{fig:bid_hist_deblur}.
\begin{figure}[h!]	
\centering
\hspace{-.2in}
\begin{minipage}{0.3\textwidth}
\includegraphics[width=1\textwidth]{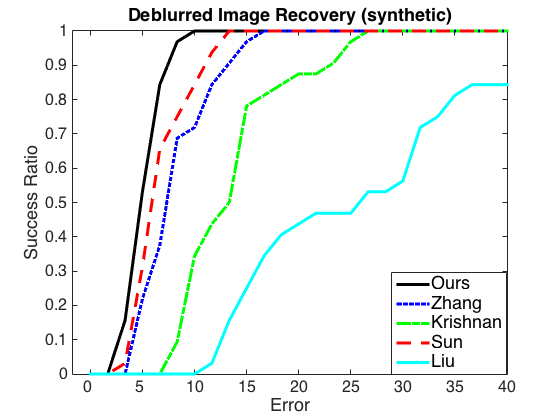}
\end{minipage}
\begin{minipage}{0.3\textwidth}
\includegraphics[width=1\textwidth]{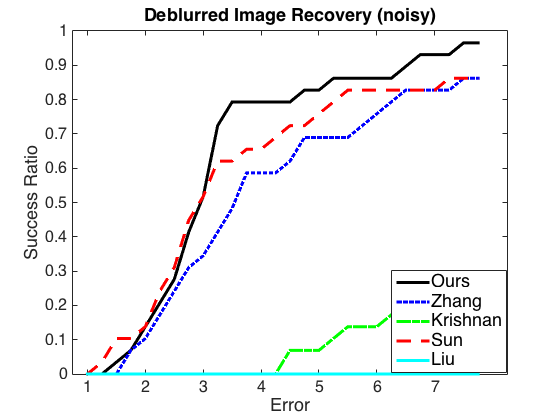}
\end{minipage}
\begin{minipage}{0.3\textwidth}
\includegraphics[width=1\textwidth]{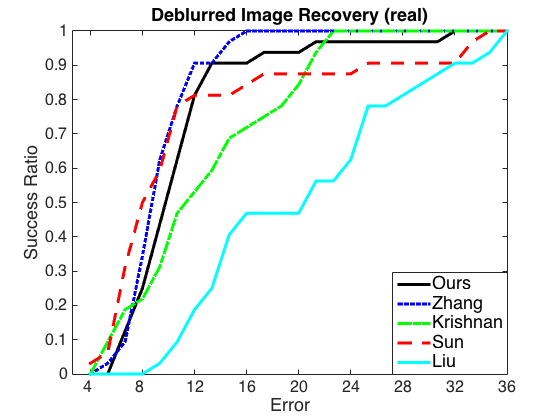}
\end{minipage}
\vspace{-.1in}
\caption{{\bf Non-blind Restoration Error}: Cumulative distributions of deblurred image error from synthetic (left), noised(middle) and real (right) blurry images.}
\label{fig:bid_hist_deblur}
\vspace{-.1in}
\end{figure}

Our algorithm achieves better convolutional kernel recovery for all three types of images, but its improvement on deblurred image is less obvious, especially for real images. This could be due to (i) the convolutional kernel in this dataset is not strictly uniform across the image, and (ii) the non-blind deconvolution algorithm exploits the heavy-tailed distribution of a natural image's gradient and becomes less sensitive to the accuracy of the recovered convolutional kernel.

\subsection{Convolutional Dictionary Learning} 
\label{subsec:CDL-exp}
We show results of recovering multiple convolutional kernels on both synthetic data (Figure \ref{fig:syn_multiBD}) and real STM data (Figure \ref{fig:real_multiBD}).
In the synthetic data, the three convolutional kernels are of size $16\times16$ and their corresponding activation signals are generated through a Bernoulli model of sparsity $0.005$. Results of both stages of the algorithm are shown in Figure \ref{fig:syn_multiBD}: the first stage returns kernels close to some shift truncations of the ground truth, and the second stage recovers the ground truth on a higher dimensional space.

\begin{figure}[h!]
\centering
\begin{minipage}{2in}
\includegraphics[width=2in,trim={.3in .2in 0 0},clip]{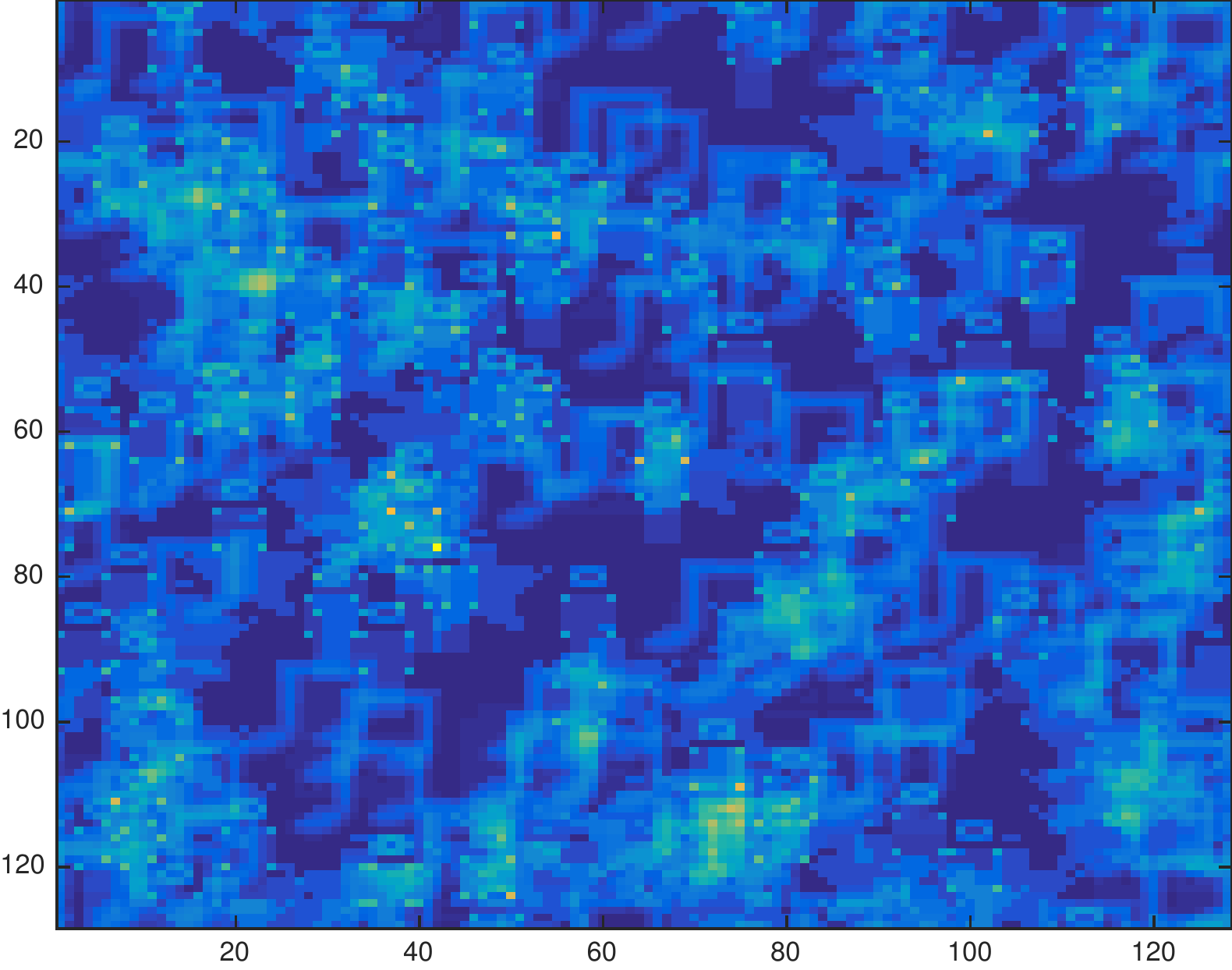}
\end{minipage}
\begin{minipage}{2.4in}
\includegraphics[width=2.45in,trim={.3in 2.45in 0in .6in},clip]{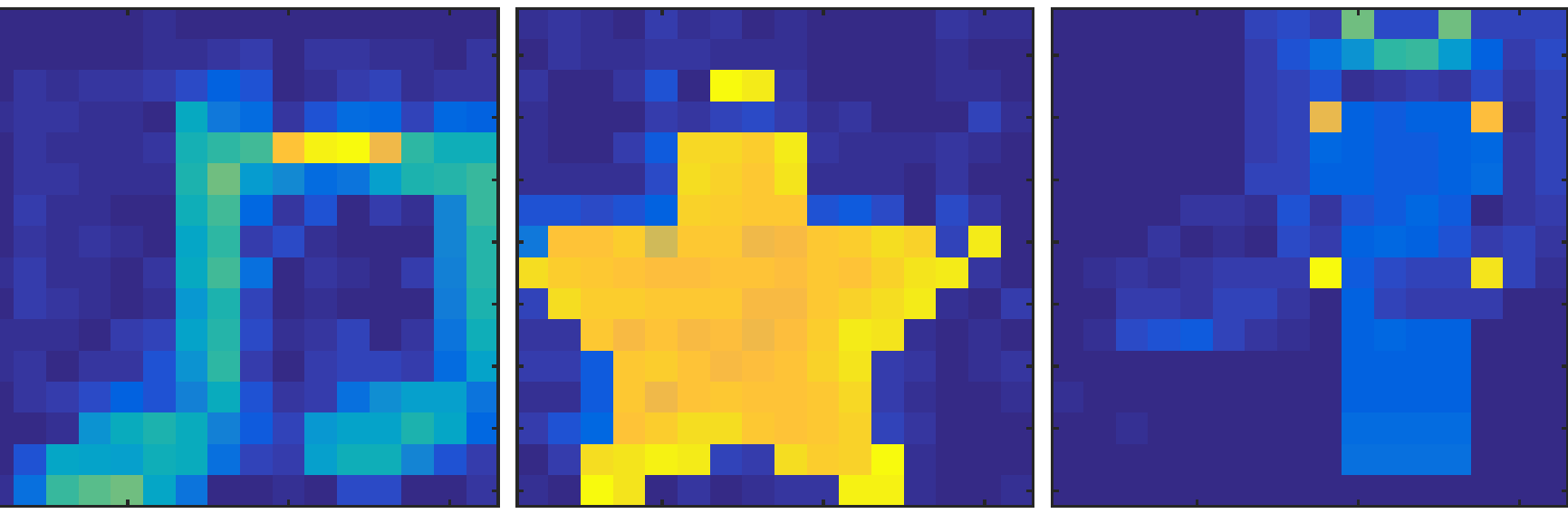}
\includegraphics[width=2.45in,trim={.3in 2.45in 0in .6in},clip]{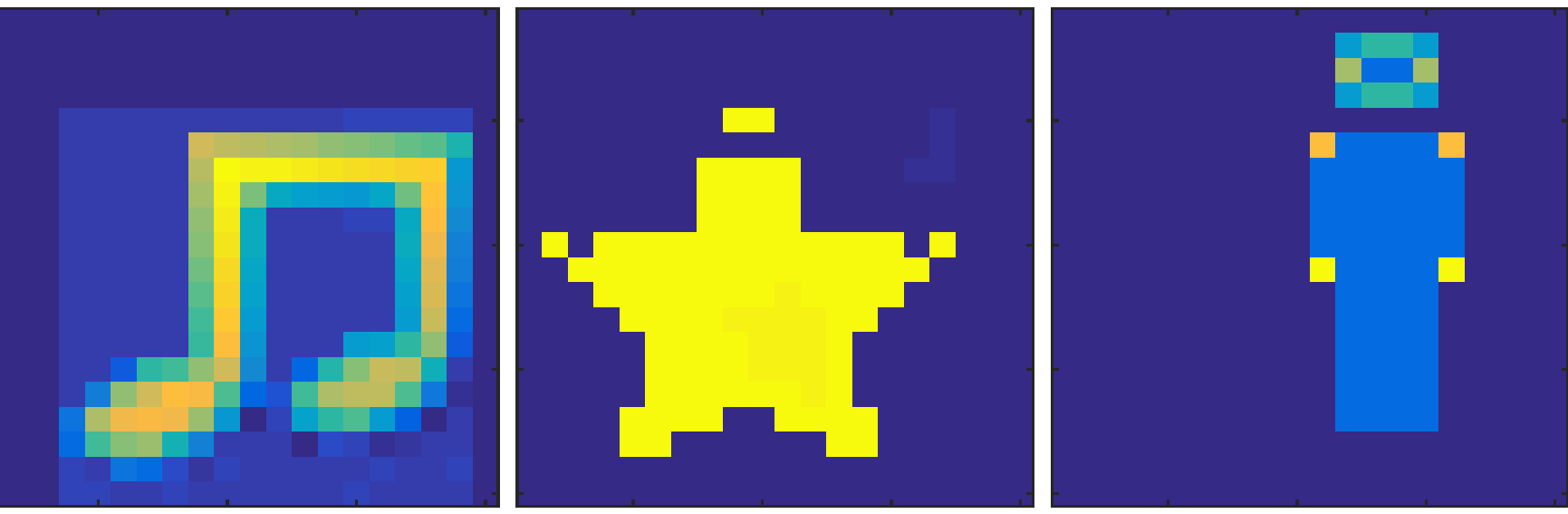}
\end{minipage}
\vspace{-.05in}
\caption{{\bf Multi Kernel Blind Deconvolution on Synthetic Data}: Input image (left) and the recovered convolutional kernels of Stage I and Stage II of the algorithm (right). }
\label{fig:syn_multiBD}
\end{figure} 

We repeat this experiment with microscopy data obtained from a NaFeAs sample. The algorithm manages to differentiate the two convolutional kernels (defect patterns), as shown in Figure \ref{fig:real_multiBD}.  For this material, the kernel orientations depend on the history of the material (stress, temperature, etc.), and using convolutional dictionary learning can be used to automatically detect these features.

\begin{figure}[h!]
\centering
\begin{minipage}{2in}
\includegraphics[width=2in]{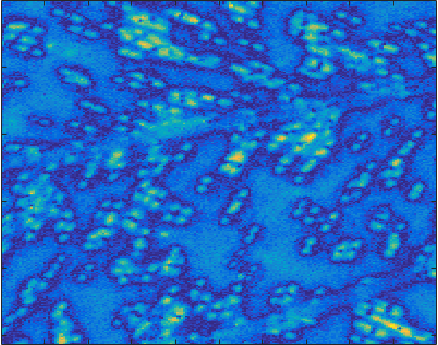} 
\end{minipage}
\begin{minipage}{2.4in}
\includegraphics[width=2.48in]{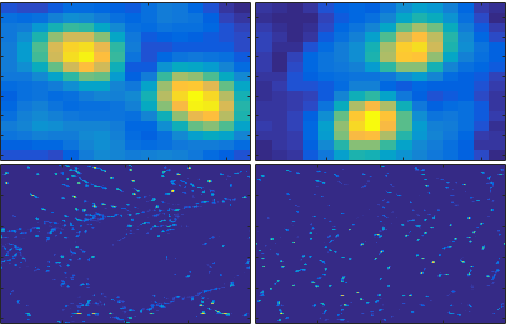}
\end{minipage}
\vspace{-.05in}
\caption{{\bf Multi Kernel Blind Deconvolution on Real STM Image}: Input image (left) and recovered convolutional kernels and their corresponding activation signals (right).}
\label{fig:real_multiBD}
\end{figure}

\section{Generalizations: Matching Loss and Constraints}

In the Lasso-like objective function studied in this paper,  the approximation error of $\mb y-\mb a\cconv\mb x$ are measured in the squared Frobenius norm (or entry-wise $\ell^2$ norm), which is usually adopted to penalize gaussian addictive noise in the observation. To obtain reliable solutions, a spherical constraint on the short kernel turned out to be crucial. Moreover, several theoretical papers \cite{Zhang18-IT, Kuo18-DQ} studying variations of the Lasso-like objective function also adopt the unit Frobenius constraint, and show that the claimed geometry continues to hold under much more general conditions. However, noise could be more complicated and does not always satisfy the gaussian model in real applications. If given prior knowledge of the noise, other penalties for the loss term $\mb y-\mb a\cconv\mb x$ could be preferable. For example, $\|\mb y-\mb a\cconv\mb x\|_p^p$ with $p<2$ performs better to penalize heavy tailed noise, and $\|\mb y-\mb a\cconv\mb x\|_p^p$ with $p>2$ performs better control on the magnitude of the noise.

In this section, we briefly discuss the following more general formulation for the short-and-sparse deconvolution problem
\begin{align}
\label{eqn:obj_pq}
&\min\quad\psi_{p}(\mb a,\mb x)\doteq\tfrac{1}{p}\norm{\mb y-\mb a\cconv\mb x}p^p+\lambda r\norm{\cdot}1\\
&\st\quad\norm{\mb a}q=1.\nonumber
\end{align}
Similarly, we write $\varphi_p(\mb a)\doteq \min_{\mb x} \psi_p(\mb a,\mb x)$. In this case, we use entrywise $\ell^p$ norm to measure the approximation error of $\mb a\cconv\mb x$, and assume the optimization constraint for $\mb a$ to be a unit $\ell^q$ ($2<q<\infty$) sphere. Comprehensive study of the function landscape of above general formulation is even more challenging. However, we can demonstrate that once $p=q$, there exist local solutions that share the same geometric property (a local optimum is close to some scaled shifted truncation) through local analysis. As in the objective optimization problem (\ref{eqn:obj}), the inherent shift ambiguity leads to these structured local minimizers. 
Because of the randomness in $\mb x_0$, the derived results hold for more general $\mb x_0$.
\begin{lemma}
Suppose $\mb y=\mb a_0\cconv\mb x_0$ with $\mb x_0=\mb e_0$, and $p=q\ge 2$. Then for any shift $\tau$, positive scalar $\eps$ and $\lambda$ such that at every $\mb a\in \bb S_q\cap\bb B\paren{\frac{ \injector^*\shift{\mb a_0}{-\tau} }{\norm{ \injector^*\shift{\mb a_0}{-\tau} }q},\eps}$ the solution $\mb x^*(\mb a)\doteq \arg\min_{\mb x} \psi_p(\mb a,\mb x)$ is (i) unique and (ii) is supported on the $\tau$-th entry, the point $\bar{\mb a}\doteq\frac{ \injector^*\shift{\mb a_0}{-\tau} }{\norm{ \injector^*\shift{\mb a_0}{-\tau} }q}$ is a strict local minimizer of the cost $\varphi_p(\mb a) $ over the manifold $\norm{\mb a}p=1$.
\end{lemma}
\begin{proof} 
Please refer to the supplement.
\end{proof}

This lemma says that {\em as long as the constraint matches the loss}, i.e., we choose $q = p \ge 2$, a scaled shift truncation $\bar{\mb a} = \injector^*\shift{\mb a_0}{-\tau}/\norm{\injector^*\shift{\mb a_0}{-\tau}}{q}$ achieves the local minimum. In contrast, {\em when the constraint does not match the loss, i.e., $q\neq p$, this property is not satisfied}: the Riemannian gradient vanishes at a stationary point $\mb a$, which satisfies
\begin{equation}
\label{eqn:p_q_lcoal}
\sign\paren{\injector^*\shift{\extend{\mb a_0}}{-\tau}-\alpha\mb a} = \sign\paren{\mb a},
\end{equation}
and
\begin{equation}
\paren{\injector^*\shift{\extend{\mb a_0}}{-\tau}-\alpha\mb a}^{\circ\paren{p-1}} = \alpha'\mb a^{\circ\paren{q-1}},
\end{equation}
with $\alpha'$ denoting another non-zero scalar of arbitrary value. This point is not a shift truncation of the ground truth.\footnote{When $p\neq q$, with $\lambda$ decreasing, \eqref{eqn:p_q_lcoal} can be rewritten as $\alpha\mb a = \injector^*\shift{\extend{\mb a_0}}{-\tau}+\alpha'\mb a^{\circ\frac{q-1}{p-1}}$. With the stationary condition for $\mb x^*(\mb a)$, we have $\alpha\mb a \to \injector^*\shift{\extend{\mb a_0}}{-\tau}$ as $\lambda\to 0$. Algorithmically, this helps to explain why practical algorithms solving \eqref{eqn:obj_pq} with $p\neq q$ can sometimes recover the kernel, where a careful continuation in $\lambda$ is always necessary.}

Hence, when the constraint matches the loss, there is good {\em local} geometry. In this paper, we have empirically demonstrated that when $p = q =2$, there is good {\em global} geometry, under appropriate conditions. We can potentially leverage this property to find global solutions to the general $\ell^p$ ($p \ge 2$) problem: one first minimizes $\ell^2$ to obtain a point near the ground truth and then locally minimizes $\varphi_p$ over an $\ell^p$ ball to obtain an estimate that uses the statistical characteristics of the noise. 

Although our analysis only pertains to $q \ge 2$, we believe that there is good local geometry even when $q<2$. In this situation, the constraint set is no longer twice differentiable and the shift truncations are always on the non-smooth region of the constraint set. As the $p$ decrease, the normal cone at the nonsmooth points gets wider. For example, if $q\ge 2$ the normal cone is the normal direction. Therefore, as $q$ decreases, the normal cone is larger and hence more like to contain the gradient direction, then this shift truncation is likely to be a stationary point.

\section{Discussions}
This work studies the global geometry of a nonconvex optimization problem for SBD when the kernel is assumed to have unit Frobenius norm. In this setting, we find that all the local minima are benign, in the sense that they are close to some signed shift truncation of the ground truth. With this insight, we propose a two stage algorithm that recovers the ground truth by exploiting the information hidden in local minima. 

This problem reveals the challenges faced when analyzing the SBD problem via a geometrical approach. For problems enjoying stronger symmetry properties \cite{SQW15-pp, SQW16-pp}, similar approaches yield a global understanding of the function geometry and recovery guarantees. We expect that the weak symmetries in SBD make a major contribution to the difficulties encountered for this problem.

There are lots of additional further directions could be of great interests for both theory and application:
Our empirical results show that the our characterization of local minima carries through to the convolutional dictionary learning problem, which can also be efficiently solved by slight adaptation of the proposed algorithm. However, the theory part is open, it would be interesting to know how many kinds of kernels, or what kinds of kernels are recoverable, probably by some measures of incoherence, which is a common assumption in dictionary learning problem.
 
Two other imperfections we encounter in scientific measurement are resolution limit and measurement error, which inspire us to consider (i) if it's possible to integrate blind deconvolution and super-resolution process together; (ii) if we can come up with a robust blind deconvolution algorithm to automatically rule out noisy entries.

\section*{Acknowledgement}
The authors gratefully acknowledge support from NSF 1343282, NSF CCF 1527809, and NSF IIS 1546411. 

{\small
\bibliographystyle{alpha}
\bibliography{deconv}
}

\newpage
\appendix
\section{Notations}
This supplementary material contains a complete proof of Theorem 2.1 and Lemma 3.1 of the main paper. Section 1 discusses some technical notations and  geometric observations, Section 2 and 3 presents a complete proof for the theorem and lemma respectively, and Section 4 contains two auxiliary lemmas used in the main proof.

Before commencing our main proof, we introduce some notations, recap the simplified objective function $\phihat$ and a few basic observations on its geometry. In this note, we use $\convmtx{\mb a}$ or $\convmtx{\extend{\mb a}}$ exchangeably to denote the $m\times m$ circulant matrix generated by a $k$-length short convolutional kernel $\mb a$ without ambiguity.

\textbf{Reversed Circulant Matrix}
Let $\revmtx$ be the reversed circulant matrix for $\mb a_0$ such that
\begin{equation} \revmtx = \left[ \begin{array}{c|c|c|c|c} \shift{\extend{\mb a_0}}{0} & \shift{\extend{\mb a_0}}{-1} & \shift{\extend{\mb a_0}}{-2} & \dots & \shift{\extend{\mb a_0}}{-(m-1)} \end{array} \right] \in \R^{m \times m},
\end{equation}
Here, $\shift{\extend{\mb a_0}}{\tau}$ denotes a cyclic shift as defined in Equation (3) of the main paper.
Since the $i$-th entry of $\convmtx{\mb a}^* \extend{\mb a_0}$ satisfies
\begin{equation}
\left[ \convmtx{\mb a}^* \extend{\mb a_0} \right]_i =\innerprod{ \shift{\extend{\mb a} }{i}  }{ \extend{\mb a_0} }
=\innerprod{ \extend{\mb a} }{ \shift{ \extend{\mb a_0 } }{-i} } 
=\innerprod{ \mb a }{ \injector^* \shift{\extend{\mb a_0}}{-i} },
\end{equation}
therefore following equation always holds
\begin{equation}
\convmtx{\mb a}^* \extend{\mb a_0} = \revmtx^* \extend{\mb a}.
\end{equation}

\textbf{Projection onto $I$}
Let $\mb e_0, \dots, \mb e_{m-1}$ denote the standard basis vectors. For index set $I = \set{ i_1, \dots, i_{|I|} } \ne \emptyset$, let
\begin{equation}
\mb V_I \doteq \left[ \mb e_{i_1} \mid \mb e_{i_2} \mid \dots \mid \mb e_{i_{|I|}} \right] \in \R^{m \times |I|},
\end{equation}
and then
\begin{equation}
\mb P_I = \mb V_I \mb V_I^* 
\end{equation}
denote the projection onto the coordinates indexed by $I$ while setting other entries zero. 

\textbf{Partial Signed Support} 
Let $\mb u$, $\mb v$ be two vectors of the same dimension. If $\supp{\mb u}\subseteq\supp{\mb v}$ and $\mb u(i)\mb v(i)\ge0$ for all $i\in\supp{\mb v}$, then $\mb u$ attains the partial signed support of $\mb v$, denoted as $\mb u\unlhd\mb v$.

\textbf{Piecewise Quadratic Function}
Let $\mb x^*(\mb a)$ denote the minimizer for simplified objective function
\begin{equation}
\phihat(\mb a) =  \min_{\mb x} \tfrac{1}{2}\norm{\extend{\mb a_0}}{2}^2 +  \tfrac{1}{2} \norm{\mb x}{2}^2 - \innerprod{\mb a\cconv\mb x }{\extend{\mb a_0}} + \lambda\norm{\mb x}1,
\end{equation}
with sign $\mb \sigma$ and support $I$ defined as
\begin{equation}
\mb \sigma \doteq \sign{ \mb x^*(\mb a) } \in \set{ -1, 0, 1 }^m, \quad I \doteq \supp{\mb \sigma}\subseteq \set{ 0, 1, \dots, m-1}\footnote{As in the main paper, we assume that $m$ dimensional vectors are indexed by the integers $0,1,\dots, m-1$.}. 
\end{equation}

By stationary condition for $\mb x^*(\mb a)$, we obtain
\begin{equation}
\mb x^*(\mb a)=\soft{\convmtx{\mb a}^* \extend{\mb a_0}}{\lambda} = \soft{\revmtx^* \injector \mb a}{\lambda},
\end{equation}
where $\soft{u}{\lambda} = \sign(u) \max\set{ |u| - \lambda, 0 }$ is the entry-wise soft-thresholding operator. 

For each sign pattern $\mb \sigma =\supp{\mb x^*}\in \set{ -1, 0, 1 }^m$, there exists corresponding region on the sphere such that 
\begin{equation}
R_{\mb \sigma} = \set{ \mb a \mid \sign{\soft{\revmtx^* \injector \mb a }{\lambda} } = \mb \sigma }.
\end{equation}
On the relative interior of each $R_{\mb \sigma}$, the function $\phihat$ has a simple expression:
\begin{equation}
\phihat(\mb a) = \phihat_{\mb \sigma}(\mb a) \;\doteq\; - \tfrac{1}{2} \mb a^* \injector^* \revmtx \mb P_I \revmtx^* \injector \mb a + \lambda \mb \sigma^* \mb P_I \revmtx^* \injector \mb a+\tfrac{1}{2} - \tfrac{\lambda^2 |I|}{2} .
\end{equation}
Therefore, the objective function is piecewise quadratic and can be rewritten as 
\begin{equation}
\phihat_{\mb \sigma} = \tfrac{1}{2} \mb a^* \mb M_{\mb \sigma} \mb a + \mb b_{\mb \sigma}^* \mb a + Const_{\mb \sigma},
\end{equation}
with
\begin{equation}
\mb M_{\mb \sigma} =-\injector^* \revmtx \mb P_I \revmtx^*\injector ,\quad  \mb b_{\mb \sigma}=\lambda\injector ^* \revmtx\mb P_I\mb \sigma.
\end{equation}
With above notations clarified, we are now ready to present a proof for the main theorem.

\section{Proof of Theorem 2.1}
\begin{proof} On the relative interior of each $R_{\mb \sigma}$, the simplified objective function
\begin{equation}
\phihat_{\mb \sigma} = \tfrac{1}{2} \mb a^* \mb M_{\mb \sigma} \mb a + \mb b_{\mb \sigma}^* \mb a + c_{\mb \sigma}
\end{equation}
has Euclidean derivative and Hessian
\begin{eqnarray}
\nabla \phihat_{\mb \sigma}(\mb a ) &=& \mb M_{\mb \sigma} \mb a + \mb b_{\mb \sigma}, \\
\nabla^2 \phihat_{\mb \sigma}(\mb a) &=& \mb M_{\mb \sigma}.
\end{eqnarray}
As we assume $\mb a$ to have unit Frobenius norm, or to live on a sphere, the more natural Riemannian gradient and Hessian are defined as
\begin{eqnarray}
\grad{\phihat_{\mb \sigma}}{\mb a} 
&=& \mb P_{\mb a^\perp}\nabla \phihat_{\mb \sigma}(\mb a ) \\
&=& \mb M_{\mb \sigma} \mb a + \mb b_{\mb \sigma} - \mb a (\mb a^* \mb M_{\mb \sigma} \mb a + \mb b_{\mb\sigma}^* \mb a),\\
\hess{\phihat_{\mb \sigma}}{\mb a} 
&=&\mb P_{\mb a^\perp} \Bigl(\nabla ^2\phihat_{\mb \sigma}(\mb a ) -\innerprod{\nabla \phihat_{\mb \sigma}(\mb a )}{\mb a}\mb I\Bigr) \mb P_{\mb a^\perp}\\
&=&\mb P_{\mb a^\perp} \Bigl( \mb M_{\mb \sigma} - (\mb a^* \mb M_{\mb \sigma} \mb a + \mb b_{\mb \sigma}^* \mb a  ) \mb I \Bigr) \mb P_{\mb a^\perp}.
\end{eqnarray}
Here, $\mb P_{\mb a^\perp}=\mb I-\mb a\mb a^*$ denotes projection onto 
the tangent space over the sphere at $\mb a$. As in the Euclidean space, a stationary point on the sphere needs to satisfy $\grad{\phihat_{\mb \sigma}}{\mb a} =\mb 0$. At a stationary point $\bar{\mb a}$, if $\hess{\phihat_{\mb \sigma}}{\bar{\mb a}} $ is positive semidefinite, the function is convex and $\bar{\mb a}$ is a local minimum; if $\hess{\phihat_{\mb \sigma}}{\bar{\mb a}}$ has a negative eigenvalue, then there exists a direction alone which the objective value decreases and hence $\bar{\mb a}$ is a saddle point \cite{Absil2007}.

Let $I = \Brac{i_1 < i_2 < \dots < i_{|I|}}$ and define
\begin{eqnarray}
\eta_i &=& \norm{ \injector^* \shift{\extend{\mb a_0}}{-i} }{2}\quad \forall i\in I,\\
\mb \eta &=& ( \eta_{i_1}, \eta_{i_2}, \dots, \eta_{i_{|I|}} ) \in \R^{|I|},
\end{eqnarray}
and 
\begin{equation}
\mb U = \left[ \frac{\sigma_{i_1} \injector^* \shift{\extend{\mb a_0}}{-i_1} }{ \eta_{i_1} } \middle| \frac{\sigma_{i_2} \injector^* \shift{\extend{\mb a_0}}{-i_2} }{ \eta_{i_2} } \middle| \dots \middle| \frac{\sigma_{i_{|I|}} \injector^* \shift{\extend{\mb a_0}}{-i_{|I|}} }{ \eta_{i_{|I|}} } \right] \in \R^{k \times |I|}. 
\end{equation}
Here, columns of $\mb U$ have unit $\ell^2$ norm. Then we have
\begin{equation}
\mb M_{\mb \sigma} = - \mb U \mathrm{diag}( \mb \eta )^2 \mb U^*,\quad\mb b_{\mb \sigma} = \lambda \mb U \mb \eta.
\end{equation}

As $I$ is defined via soft thresholding, we have $\abs{ [ \revmtx^* \injector \mb a ]_i } > \lambda$ holds for every $i \in I$. Hence,
\begin{eqnarray}
\mb a^* \mb M_{\mb \sigma} \mb a + \mb b_{\mb \sigma}^* \mb a
&=& -\mb a^* \injector^* \revmtx \mb P_{I} \revmtx^* \injector \mb a  + \lambda \mb \sigma^* \mb P_{I} \revmtx^* \injector \mb a \\
&=& -\norm{\mb P_I \revmtx^* \injector \mb a }{2}^2 + \lambda \norm{ \mb P_I \revmtx^* \injector \mb a }{1} \\
&<& 0
\end{eqnarray}
holds at any $\mb a\in\cl{R_{\mb \sigma}}\setminus R_{\mb 0}$.

\textbf{Stationary point and implications}

Consider any stationary point $\bar{\mb a} \in \cl{ R_{\mb \sigma} } \setminus R_{\mb 0}$ of $\phihat$. By continuity of the gradient of $\phihat$ (proved in Lemma \ref{lem:continuity}), $\bar{\mb a}$ is also a stationary point of $\phihat_{\mb \sigma}$. By definition, $\grad{\phihat_{\mb \sigma}}{\bar{\mb a}} = \mb 0$, which implies that 
\begin{equation}
(\bar{\mb a}^* \mb M_{\mb \sigma} \bar{\mb a} + \mb b_{\mb \sigma}^* \bar{\mb a}) \bar{\mb a} = \mb M_{\mb \sigma} \bar{\mb a} + \mb b_{\mb \sigma}.
\end{equation}
Note that since $\bar{\mb a}^* \mb M_{\mb \sigma} \bar{\mb a} + \mb b_{\mb \sigma}^* \bar{\mb a} \ne 0$ and $\mb b_{\mb \sigma} \in \mathrm{range}(\mb M_{\mb \sigma})$, this implies that $\bar{\mb a} \in \mathrm{range}(\mb M_{\mb \sigma})$.

Let $\gamma = -\paren{ \mb a^* \mb M_{\mb \sigma} \mb a + \mb b_{\mb \sigma}^* \mb a}>0$, then the condition for a stationary point $\bar{\mb a}$ becomes 
\begin{equation}
\gamma \bar{\mb a} = \mb U \mathrm{diag}(\mb \eta)^2 \mb U^* \bar{\mb a} - \lambda \mb U \mb \eta. 
\end{equation}
Let $\mb \alpha = \mb U^* \bar{\mb a}$, and note that for each $j$, $\alpha_j > 0$ and $\alpha_j \eta_j > \lambda$. In terms of $\mb U$, the stationarity condition becomes
\begin{equation}
\gamma \mb \alpha = \mb U^* \mb U \mathrm{diag}(\mb \eta)^2 \mb \alpha - \lambda \mb U^* \mb U \mb \eta. 
\end{equation}
Since the diagonal elements of $\mb U^* \mb U$ are all ones, and hence can be written as
\begin{equation}
\mb U^* \mb U = \mb I + \mb \Delta.
\end{equation}
We have
\begin{equation}
\gamma \mb \alpha = \diag(\mb \eta)^2 \mb \alpha - \lambda \mb \eta + \mb \Delta \Bigl\{ \diag(\mb \eta)^2 \mb \alpha - \lambda \mb \eta \Bigr\}.
\end{equation}
As $\mb\alpha\succ\mb0$ and $\mathrm{diag}(\mb \eta) \mb \alpha\succ\lambda\cdot\mb1$\footnote{Here, $\succ$ denotes element-wise inequality between vectors.}, together with an auxiliary Lemma \ref{lem:orth-perturb} proved in Section \ref{sec:lemmas}, we have
\begin{eqnarray}
\norm{ \mathrm{diag}(\mb \eta)^2 \mb \alpha - \lambda \mb \eta }{2} &\le& \norm{\mb \alpha}{2} \\ &=& \norm{\mb U^* \bar{\mb a}}{2} \\ &\le& \norm{\mb U}{\ell^2 \to \ell^2} \\
&\le& \sqrt{3/2}, 
\end{eqnarray}
whence
\begin{eqnarray}
\norm{ \mb \Delta \Bigl\{ \diag( \mb \eta )^2 \mb \alpha - \lambda \mb \eta \Bigr\} }{ \infty } &\le& \sqrt{3/2} \times \norm{\mb \Delta }{\ell^2 \to \ell^{\infty}}.  \label{eqn:res-bound}
\end{eqnarray}
Suppose that $\mb \Delta$ is small enough that the right hand side of \eqref{eqn:res-bound} is bounded by $\lambda^2 / 2$, i.e., 
\begin{equation} \label{eqn:2-inf-bound}
\norm{\mb \Delta}{\ell^2 \to \ell^\infty} \le \frac{\lambda^2}{\sqrt{6}}.
\end{equation}
Plugging back into the stationary condition $\diag(\mb \eta)^2 \mb \alpha-\gamma \mb \alpha =  \lambda \mb \eta - \mb \Delta \Bigl\{ \diag(\mb \eta)^2 \mb \alpha - \lambda \mb \eta \Bigr\}$ gives
\begin{equation}
( \diag(\mb \eta)^2 - \gamma ) \mb \alpha \succ \lambda \mb \eta - \lambda^2 / 2 \succ \mb 0.
\end{equation}
Since $\alpha_i<1$ and $\eta_i>\lambda$ for all $i$, which implies that
\begin{equation}
\gamma < \eta_{\min}^2 - \lambda \eta_{\min} + \lambda^2/2, \label{eqn:gamma-bound}
\end{equation}
where $\eta_{\min}$ is the smallest of the $\eta_i$. 

\textbf{Negative curvature in Hessian} Recall the Riemannian Hessian on the sphere is defined as
\begin{eqnarray}
\hess{\phihat_{\mb \sigma}}{\mb a} &=& \mb P_{\mb a^\perp} \Bigl( \mb M_{\mb \sigma} - (\mb a^* \mb M_{\mb \sigma} \mb a + \mb b_{\mb \sigma}^* \mb a  ) \mb I \Bigr) \mb P_{\mb a^\perp}\\
&=& \mb P_{\mb a^\perp} \Bigl( -\mb U \diag(\mb \eta)^2 \mb U^* +\gamma \mb I \Bigr) \mb P_{\mb a^\perp}.
\end{eqnarray}
Below argument shows that this Riemmanian Hessian has negative eigenvalues.
Let $\tilde{\mb U}$ be an orthonormal matrix generated via
\begin{equation}
\tilde{\mb U} \doteq \mb U (\mb U^* \mb U)^{-1/2}
\end{equation}
Whenever $\norm{\mb \Delta}{\ell^2 \to \ell^2 } < 1/2$ holds, Lemma \ref{lem:orth-perturb} guarantees $\norm{\mb U - \tilde{\mb U}}{\ell^2 \to \ell^2} < 3 \norm{\mb \Delta}{\ell^2 \to \ell^2}$. Under this condition, we can lower bound the smallest nonzero eigenvalue of $\mb U \diag(\mb \eta)^2 \mb U^*$, as 
\begin{eqnarray}
\lambda_{\min}( \mb U \diag(\mb \eta)^2 \mb U^* ) &=& \sigma_{\min}( \mb U \diag(\mb \eta) )^2 \\
 &\ge& \left( \max\set{ \sigma_{\min}\left( \tilde{\mb U} \diag(\mb \eta) \right) - \norm{ \mb U - \tilde{\mb U}}{} \norm{\diag(\mb \eta)}{} , 0 } \right)^2 \\
 &=& \left( \max \set{ \eta_{\min} - \norm{\mb U - \tilde{\mb U}}{} \eta_{\max}, 0 } \right)^2 \\
 &\ge& \eta_{\min}^2 - 3 \eta_{\max}  \eta_{\min} \norm{ \mb \Delta }{ \ell^2 \to \ell^2 }. 
\end{eqnarray}
Since $\lambda<\eta_{\min}\le\eta_{\max} \le 1$, additionally if $\norm{\mb \Delta}{\ell^2 \to \ell^2} \le \frac{\lambda}{6}$, we have
\begin{equation} 
\label{eqn:delta-cond-1}
3 \eta_{\max} \eta_{\min} \norm{\mb \Delta}{\ell^2 \to \ell^2} \le \lambda \eta_{\min} - \lambda^2 / 2, 
\end{equation}
Together with \eqref{eqn:2-inf-bound} and \eqref{eqn:gamma-bound}, we can obtain
\begin{equation}
\lambda_{\min}(\mb U \diag(\mb \eta)^2 \mb U^* ) > \gamma,
\end{equation}
or
\begin{equation}
\lambda_{\max}( \mb M_{\mb \sigma} ) < - \gamma. 
\end{equation}
Thus, whenever the following conditions are satisfied
\begin{equation}
\lambda < 1, \quad \norm{\mb \Delta }{\ell^2 \to \ell^\infty} \le \frac{\lambda^2}{\sqrt{6}}, \quad \norm{ \mb \Delta }{\ell^2 \to \ell^2} \le \frac{\lambda}{6},
\end{equation}
we have $\lambda_{\max}(\mb M_{\mb \sigma}) < - \gamma$ as desired.

Above calculations imply that for every $\mb \xi \in \range{ \mb M_{\mb \sigma} } \subseteq \R^k$,
\begin{equation}
\mb \xi^* \left( \mb M_{\mb \sigma} + \gamma \mb I \right) \mb \xi < 0. 
\end{equation}
Since
\begin{equation}
\hess{\phihat_{\mb \sigma}}{\bar{\mb a}} = \mb P_{\bar{\mb a}^\perp} \Bigl( \mb M_{\mb \sigma} + \gamma \mb I \Bigr) \mb P_{\bar{\mb a}^\perp},
\end{equation}
for $\mb \xi \in \bar{\mb a}^\perp \cap \range{ \mb M_{\mb \sigma} }$, 
\begin{equation}
\mb \xi^* \hess{\phihat_{\mb \sigma}}{\bar{\mb a}} \mb \xi < 0.
\end{equation}
Hence, on the relative interior $\relint{ R_{\mb \sigma} }$, $\phihat \equiv \phihat_{\mb \sigma}$ obtains, and so this implies that for $\norm{\mb \sigma}{0} > 1$, there are no local minima in $\relint{R_{\mb \sigma}}$.

\textbf{Relative boundaries}
We first note that if $\norm{\mb \sigma}{0} = 1$ and $I=\set{i}$, either $R_{\mb \sigma}$ is empty when $\norm{\injector^* \shift{\extend{\mb a_0}}{-i} }{2}\le\lambda$, or it contains an open ball around $\range{\mb M_{\mb \sigma}} \bigcap \bb S^{k-1}=\pm\frac{\injector^* \shift{\extend{\mb a_0}}{-i} }{ \norm{\injector^* \shift{\extend{\mb a_0}}{-i} }{2} }$. Hence, if $\bar{\mb a} \in \rbdy{ R_{\mb \sigma} }$ is a stationary point and $\mb \sigma \ne \mb 0$, we necessarily have $\norm{\mb \sigma}{0} \ge 2$. 

Since $\bar{\mb a}$ is on the boundary of $R_{\mb \sigma}$, it is also in $\rbdy{ \cl{ R_{\mb \sigma'} }}$ for some $\mb \sigma' \ne \mb \sigma$. Let 
\begin{equation}
\Xi = \set{ \mb \sigma' \mid \bar{\mb a} \in \rbdy{ \cl { R_{\mb \sigma' }} } }.
\end{equation}
Suppose that for every $\mb \sigma' \in \Xi$, $\mb \sigma \unlhd \mb \sigma'$. Hence, $\range{\mb M_{\mb \sigma}} \subseteq \range{ \mb M_{\mb \sigma'}}$ for every $\mb \sigma' \in \Xi$ and
\begin{equation}
\mb \xi^* \hess{ \phihat_{\mb \sigma'} }{\bar{\mb a}} \mb \xi < 0, \quad   \forall \; \mb \xi \in \mathrm{range}(\mb M_{\mb \sigma}), \; \mb \sigma' \in \Xi. 
\end{equation}
By continuity of the gradients, $\bar{\mb a}$ is a stationary point for {\em every} $\phihat_{\mb \sigma'}$ such that $\mb \sigma' \in \Xi$. If we choose an arbitrary nonzero $\mb \xi \in \mathrm{range}(\mb M_{\mb \sigma})$, we have that for every $\mb \sigma' \in \Xi$, 
\begin{equation}
\phihat_{\mb \sigma'}(\mc P_{\bb S^{k-1}}\brac{\bar{\mb a}+t \mb \xi}  ) < \phihat(\bar{\mb a}) - \Omega(t^2). 
\end{equation}
There exists a neighborhood $N$ of $\bar{\mb a}$ for which, at every $\mb a \in N\cap R_{\mb \sigma'}$, $\phihat(\mb a) = \phihat_{\mb \sigma'}(\mb a)\le\phihat(\bar{\mb a}) $ for some $\mb \sigma' \in \Xi$. Hence, $\bar{\mb a}$ is not a local minimum of $\phihat$. 

\textbf{Local minima}
If $\norm{\mb \sigma}{0} = 1$ and $I=\set{i}$, then the simplified objective function is
\begin{equation}
\phihat_{\mb \sigma} = -\tfrac{1}{2}\innerprod{\sigma_i\injector^* \shift{\extend{\mb a_0}}{-i} }{\mb a}^2 +\lambda\innerprod{\sigma_i\injector^* \shift{\extend{\mb a_0}}{-i} }{\mb a} + c_{\mb \sigma}.
\end{equation}
The minimizer appears at the boundary for $\innerprod{\sigma_i\injector^* \shift{\extend{\mb a_0}}{-i} }{\mb a}$, namely $\norm{\injector^* \shift{\extend{\mb a_0}}{-i}}{2}$\footnote{The other boundary point is $\innerprod{\sigma_i\injector^* \shift{\extend{\mb a_0}}{-i} }{\mb a}=\lambda$, which achieves a smaller objective value.} obtained by
\begin{equation}
\bar{\mb a}  = \sigma_i \frac{ \injector^* \shift{\extend{\mb a_0}}{\tau} }{ \norm{\injector^* \shift{\extend{\mb a_0}}{\tau} }{2} }.
\end{equation}
It can be easily verified that 
\begin{eqnarray}
\grad{\phihat_{\mb \sigma}}{\bar{\mb a}} &=& \paren{-\norm{\injector^* \shift{\extend{\mb a_0}}{\tau} }{2}^2+\lambda\norm{\injector^* \shift{\extend{\mb a_0}}{\tau} }{2}}\times\paren{\mb I-\bar{\mb a}\bar{\mb a}^*}\bar{\mb a}\\
&=&\mb 0\\
\hess{\phihat_{\mb \sigma}}{\bar{\mb a}} &=& \mb P_{\mb a^\perp} \Bigl( -\norm{\injector^* \shift{\extend{\mb a_0}}{\tau} }{2}^2\bar{\mb a}\bar{\mb a}^* +(1 - \lambda \norm{\injector^* \shift{\extend{\mb a_0}}{\tau} }{2}) \mb I \Bigr) \mb P_{\mb a^\perp}\\
&=&(1 - \lambda \norm{\injector^* \shift{\extend{\mb a_0}}{\tau} }{2}) \mb P_{\mb a^\perp}\mb P_{\mb a^\perp}\\
&\succeq&0
\end{eqnarray}
\textbf{Global maxima}
If $\norm{\mb \sigma}{0} = 0$, then the objective remains constant and achieves the global maximum.

\end{proof}

\begin{lemma}[Continuity of the Gradient of  $\phihat$] 
\label{lem:continuity}
$\nabla \phihat$ is a continuous function of $\mb a$. 
\end{lemma}
\begin{proof}
Recall that for a given $\mb \sigma$, the gradient
\begin{eqnarray}
\nabla \phihat_{\mb \sigma}(\mb a ) &=& -\injector^* \revmtx \mb P_I \revmtx^*\injector\mb a+\lambda\injector ^* \revmtx\mb P_I\mb \sigma\\
&=&-\injector^* \revmtx \mb P_I \paren{\revmtx^*\injector\mb a-\lambda\mb\sigma}
\end{eqnarray}
This is a continuous function within the relative interior of $R_{\mb \sigma}$. Next, we show this function is continuous at the relative boundary of $R_{\mb \sigma}$. Let $\mb a'\in\rbdy{R_{\mb \sigma}}$, and $\mb\sigma'=\sign(\mb a')$, $I=\supp(\mb\sigma')$ are the corresponding sign and support. Without loss of generality, we assume $\mb\sigma'\unlhd\mb\sigma$, denote $\mb a=\mb a'+\eps\mb\delta$ ($\norm{\mb\delta}2=1$) and $I_{\mb\delta}=I\setminus I'$, then
\begin{eqnarray}
\nabla \phihat_{\mb \sigma}(\mb a )-\nabla \phihat_{\mb \sigma'}(\mb a' )
&=&-\injector^* \revmtx \mb P_I \paren{\revmtx^*\injector\mb a-\lambda\mb\sigma}+\injector^* \revmtx \mb P_{I'} \paren{\revmtx^*\injector\mb a'-\lambda\mb\sigma'}\\
&=&-\injector^* \revmtx (\mb P_{I'}+\mb P_{I_{\delta}})  \paren{\revmtx^*\injector\mb a-\lambda\mb\sigma}+\injector^* \revmtx \mb P_{I'} \paren{\revmtx^*\injector\mb a'-\lambda\mb\sigma'}\\
&=&-\injector^* \revmtx \mb P_{I_{\delta}}\paren{\revmtx^*\injector\mb a-\lambda\mb\sigma}-\eps\injector^* \revmtx \mb P_{I'}\revmtx^*\injector\mb\delta
\end{eqnarray}
Since $\norm{\mb P_{I_{\delta}}\paren{\revmtx^*\injector\mb a-\lambda\mb\sigma}}{\infty}=\eps\norm{\mb P_{I_{\delta}}\revmtx^*\injector\mb\delta}{\infty}$, we have $\norm{\nabla \phihat_{\mb \sigma}(\mb a )-\nabla \phihat_{\mb \sigma'}(\mb a' )}{\infty}\le\mc O(\eps)$.
\end{proof}

\section{Proof of Lemma 3.1}

\begin{lemma}
\label{lem:stage2}
Let $\lambda_{rel} = \lambda/\norm{\mb x_0}{\infty}$,
suppose the ground truth $\mb a_0$ satisfies 
\vspace{-.1in}
\begin{equation}
\abs{\innerprod{\mb a_0}{\injector\shift{\extend{\mb a_0}}{\tau\neq 0}}}<\lambda_{rel}^2-\paren{2+1/\lambda_{rel}^2}\sqrt{1-\lambda_{rel}^2}
\vspace{-.05in}
\end{equation}
 for any nonzero shift $\tau$, and $\mb x_0$ is separated enough such that any two nonzero components are at least $2k$ entries away from each other. If initialized at some $\mb a\in\bb S^{k-1}$ that $\abs{\innerprod{\mb a}{\mb a_0}}>\lambda/\norm{\mb x_0}{\infty}$,  
a small-stepping projected gradient method minimizing $\varphi(\mb a)$ recovers the signed ground truth $\pm\mb a_0$.
\end{lemma}

\begin{proof}Without loss of generality, we are going to assume $\norm{\mb x_0}{\infty} =1$ for simplicity. Given that
\begin{equation}
\abs{\innerprod{\mb a_0}{\injector\shift{\extend{\mb a_0}}{\tau\neq 0}}}<\lambda^2-\sqrt{1-\lambda^2}\paren{2+1/\lambda^2}
\end{equation}
and $\mb a=\innerprod{\mb a}{\mb a_0}\mb a_0+\mb\delta$ with $\norm{\mb\delta}2\le\sqrt{1-\lambda^2}$, therefore
\begin{eqnarray*}
\abs{\innerprod{\mb a}{\injector\shift{\mb a}{\tau}}}
&=&\abs{\innerprod{\innerprod{\mb a}{\mb a_0}\mb a_0+\mb\delta}{\injector\shift{\innerprod{\mb a}{\mb a_0}\mb a_0+\mb\delta}{\tau}}}\\
&\le&\innerprod{\mb a}{\mb a_0}^2 \abs{\innerprod{\mb a_0}{\injector\shift{\mb a_0}{\tau}}}+2\innerprod{\mb a}{\mb a_0}\norm{\mb\delta}2+\norm{\mb \delta}2^2\\
&<&1-\sqrt{1-\lambda^2}/{\lambda^2}
\end{eqnarray*}
Moreover, as $\mb x_0$ is sufficiently separated, we have
\begin{eqnarray*}
\abs{\innerprod{\mb a}{\injector\shift{\mb a}{\tau}}\norm{\mb x^\star}{\infty}-\innerprod{\mb a}{\injector\shift{\mb a_0}{\tau}}\norm{\mb x_0}{\infty}}
&\le&\abs{\innerprod{\mb a}{\injector\shift{\mb a}{\tau}}}\norm{\mb x_0-\mb x^\star}{\infty}+\abs{\innerprod{\mb a}{\injector\shift{\mb a_0-\mb a}{\tau}}}\norm{\mb x_0}{\infty}\\
&<&\lambda\abs{\innerprod{\mb a}{\injector\shift{\mb a}{\tau}}}+\norm{\mb a_0-\mb a}2\norm{\mb x_0}{\infty}\\
&<&\lambda.
\end{eqnarray*}
Hence, there exists a unique nonzero minimizer satisfying $\supp{\mb x^\star}\subset \supp{\mb x_0}$, and the optimality condition for $\mb x^\star$ implies
\begin{equation}
\mb x^\star=\soft{\innerprod{\mb a}{\mb a_0}\mb x_0}{\lambda},
\end{equation}


In this case, we can calculate the Euclidean gradient
\begin{eqnarray}
\nabla\varphi(\mb a)
&=&\injector\mb C^*_{\mb x^\star}\paren{\mb a\cconv\mb x^\star-\mb a_0\cconv\mb x_0}\\
&=&\norm{\mb x^\star}2^2\mb a-\innerprod{\mb x^\star}{\mb x_0}\mb a_0,
\end{eqnarray}
and the Riemannian gradient
\begin{eqnarray}
\grad{\varphi}{\mb a}&=&(\mb I-\mb a\mb a^*)\nabla\varphi(\mb a)\\
&=&-\innerprod{\mb x^\star}{\mb x_0}(\mb I-\mb a\mb a^*)\mb a_0.
\end{eqnarray}

It is easy to check that at any point along the geodesic curve between $\mb a_0$ and $\mb a$, support recovery of $\mb x^\star$ is achieved. A small-stepping gradient descent algorithm moves towards the signed ground truth $\pm\mb a_0$, as desired. 
\end{proof}

\section{Proof of Lemma 6.1}
\begin{lemma}
Suppose $\mb y=\mb a_0\cconv\mb x_0$ with $\mb x_0=\mb e_0$, and $p=q\ge 2$. Then for any shift $\tau$, positive scalar $\eps$ and $\lambda$ such that at every $\mb a\in \bb S_q\cap\bb B\paren{\frac{ \injector^*\shift{\mb a_0}{-\tau} }{\norm{ \injector^*\shift{\mb a_0}{-\tau} }q},\eps}$ the solution $\mb x^*(\mb a)\doteq \arg\min_{\mb x} \psi_p(\mb a,\mb x)$  is (i) unique and (ii) supported on the $\tau$-th entry, the point $\bar{\mb a}\doteq\frac{ \injector^*\shift{\mb a_0}{-\tau} }{\norm{ \injector^*\shift{\mb a_0}{-\tau} }q}$ is a strict local minimizer of the cost $\varphi_p(\mb a) $ over the manifold $\norm{\mb a}q=1$.
\end{lemma}

\begin{proof} Write $\mb x^*(\mb a)=\alpha\mb e_{\tau}$ at point $\mb a\in\bb S_p\cap\bb B\paren{\frac{ \injector^*\shift{\mb a_0}{-\tau} }{\norm{ \injector^*\shift{\mb a_0}{-\tau} }q},\eps}$, and suppose without loss of generality that $\alpha > 0$. Then using the KKT conditions, we have
\begin{align}
\lambda &= \innerprod{ \shift{\extend{\mb a}}{\tau} }{ \diag\sign\paren{\extend{\mb a_0}-\alpha\shift{\extend{\mb a}}{\tau}}\abs{\extend{\mb a_0}-\alpha\shift{\extend{\mb a}}{\tau}}^{\circ\paren{p-1}} }\\
&= \innerprod{ \mb a }{ \diag\sign\paren{\injector^*\shift{\extend{\mb a_0}}{-\tau}-\alpha\mb a}\abs{\injector^*\shift{\extend{\mb a_0}}{-\tau}-\alpha\mb a}^{\circ\paren{p-1}} }.
\end{align}

\paragraph{Magnitude of $\mb x^*$}
At point $\bar{\mb a}=\frac{ \injector^*\shift{\mb a_0}{-\tau} }{\norm{ \injector^*\shift{\mb a_0}{-\tau} }q}$, there exists closed form solution   
\begin{align}
\mb x^*\paren{\bar{\mb a}} = \bar\alpha\mb e_{\tau}=\paren{\norm{ \injector^*\shift{\mb a_0}{-\tau} }q - \lambda^{\frac1{p-1}}}\mb e_{\tau}.
\end{align}
As the shifted support recovery of $\mb x^*(\mb a)$ holds at point $\mb a\in\bb B(\bar{\mb a},\eps)$, with implicit function theorem, we can obtain that $\alpha$ is a differentiable function of $\mb a$, and 
\begin{align}
\nabla_{\mb a}\alpha\paren{\bar{\mb a}}=-\paren{\norm{ \injector^*\shift{\mb a_0}{-\tau} }q-2\lambda^{\frac1{p-1}}}\sign\bar{\mb a}\circ\abs{\bar{\mb a}}^{\circ(p-1)}.
\end{align}

\paragraph{Euclidean Gradient and Hessian}
The Euclidean gradient of $\varphi_p$ at $\bar{\mb a}$ is
\begin{align}
\nabla \varphi_p(\bar{\mb a}) &= \nabla_{\mb a}\psi_{p}\Bigr|_{(\bar{\mb a},\mb x^*(\bar{\mb a}))}\\
&= \injector^*\mb C_{\mb x^*(\bar{\mb a})}^*\diag\sign\paren{\bar{\mb a}\cconv\mb x^*(\bar{\mb a})-\mb y}\abs{\bar{\mb a}\cconv\mb x^*(\bar{\mb a})-\mb y}^{\circ\paren{p-1}}\\
& = \bar\alpha\injector^*\diag\sign\paren{\shift{\bar\alpha\shift{\bar{\mb a}}{\tau}-\extend{\mb a_0}}{-\tau}}\abs{\shift{\bar\alpha\shift{\bar{\mb a}}{\tau}-\extend{\mb a_0}}{-\tau}}^{\circ\paren{p-1}}\\
& = \bar\alpha\diag\sign\paren{\bar\alpha\bar{\mb a}-\injector^*\shift{\extend{\mb a_0}}{-\tau}}\abs{\bar\alpha\bar{\mb a}-\injector^*\shift{\extend{\mb a_0}}{-\tau}}^{\circ\paren{p-1}}.
\end{align}
And the Euclidean Hessian of $\varphi_p$ at $\bar{\mb a}$ is
\begin{align}
\nabla^2\varphi_p (\bar{\mb a})& = (p-1)\bar\alpha \diag\abs{\bar\alpha\bar{\mb a}-\injector^*\shift{\extend{\mb a_0}}{-\tau}}^{\circ\paren{p-2}}\paren{\bar\alpha\mb I+\bar{\mb a}\paren{\nabla_{\mb a}\alpha}^T}\nonumber\\ 
& \quad  +\sign\paren{\bar\alpha\bar{\mb a}-\injector^*\shift{\extend{\mb a_0}}{-\tau}}\circ\abs{\bar\alpha\bar{\mb a}-\injector^*\shift{\extend{\mb a_0}}{-\tau}}^{\circ\paren{p-1}}\paren{\nabla_{\mb a}\alpha}^T.
\end{align}

\paragraph{Riemannian Gradient and Hessian}We call the optimization domain  manifold $\mathcal{M}_q$  where $\norm{\mb a}q=1$ with $2\le q<\infty$, the Riemannian gradient at $\mb a$ can be written as
\begin{equation}
\mathrm{grad}[\varphi_p]\paren{\mb a} = \bb P_{\mb a^\perp}\nabla\varphi_p,
\end{equation}
with 
\begin{align}
\bb P_{\mb a^\perp}(\mb v) = \mb v- \innerprod{\mb v}{\frac{\sign{\mb a}\circ\abs{\mb a}^{\circ\paren{q-1}}}{\norm{\abs{\mb a}^{\circ\paren{q-1}}}2}}\frac{\sign{\mb a}\circ\abs{\mb a}^{\circ\paren{q-1}}}{\norm{\abs{\mb a}^{\circ\paren{q-1}}}2}. 
\end{align}
Therefore, if $p=q$, $\bar{\mb a}$ is a stationary point with vanishing Riemannian gradient  
\begin{align}
\mathrm{grad}[\varphi_p]\paren{\bar{\mb a}}=\mb 0.
\end{align}

\newcommand{\mr}{\mathrm}
\newcommand{\reals}{\mathbb{R}}
The Riemannian Hessian $\mr{Hess}[\varphi_p](\mb a) : T_{\mb a} \mc M_q \to T_{\mb a} \mc M_q$ is defined via the formula
\begin{eqnarray}
\mr{Hess}[\varphi_p](\mb a) \mb \delta &=& \tilde{\nabla}_{\mb \delta} \mr{grad}[\varphi_p](\mb a) \qquad \text{($\mb \delta \in T_{\mb a} \mc M_q$)} \\ 
&=& \mb P_{T_{\mb a} \mc M} {\nabla}_{\mb \delta} \mr{grad}[\varphi_p](\mb a),
\end{eqnarray}
where in the second line we have used Proposition 5.3.2 of the \cite{absil2009}. Here, $\tilde{\nabla}$ is the Riemannian connection on $\mc M_q$,
 $\nabla$ is the standard Euclidean connection (directional derivative) on $\reals^k$ and $\mr{grad}[\varphi_p]$ is any smooth extension of the Riemannian gradient to $\reals^k$. Plugging in and using the definition of Euclidean gradient $\nabla$, we have
\begin{eqnarray}
\mr{Hess}[\varphi_p](\mb a) \mb \delta &=& \bb P_{\mb a^\perp} \nabla_{\mb \delta} \left[ \bb P_{\mb a^\perp} \nabla \varphi_p(\mb a) \right] \\
&=& \bb P_{\mb a^\perp} \frac{d}{dt} \left[ \bb P_{(\mb a + t \mb \delta)^\perp} \nabla \varphi_p(\mb a + t \mb \delta) \right] \Bigr|_{t=0}  \\
&=& \bb P_{\mb a^\perp} \left( \nabla^2 \varphi_p( \mb a) \mb \delta + \frac{d}{dt} \bb P_{(\mb a + t \mb \delta)^\perp} \Bigr|_{t=0} \nabla \varphi_p(\mb a) \right)
\end{eqnarray}
For simplicity, we write 
\begin{equation}
\mb \xi(\mb a) = \frac{\mr{sign}(\mb a) \circ | \mb a|^{\circ (q-1)} }{ \norm{| \mb a |^{\circ (q-1)} }{2}},
\end{equation}
and
\begin{equation}
\bb P_{\mb a^\perp}=\mb I - \mb \xi(\mb a) \mb \xi(\mb a)^*.
\end{equation}
Notice that the vector-valued function $\mb \xi(\mb a)$ is differentiable away from $\mb a = \mb 0$, and that 
\begin{equation}
\frac{d}{dt} \bb P_{(\mb a+t \mb \delta)^\perp} \Bigr|_{t=0} =  - \frac{d}{dt} \mb \xi( \mb a+ t \mb \delta ) \Bigr|_{t=0} \mb \xi(\mb a)^* - \mb \xi(\mb a) \left[ \frac{d}{dt} \mb \xi( \mb a+ t \mb \delta ) \Bigr|_{t=0} \right]^*.
\end{equation}
Since that $\bb P_{\mb a^\perp} \mb \xi(\mb a) = \mb 0$, we have 
\begin{eqnarray}
\mr{Hess}[\varphi_p](\mb a) \mb \delta &=& \bb P_{\mb a^\perp}\left(  \nabla^2 \varphi_p(\mb a) \mb \delta -  \left[ \frac{d}{dt} \Bigr|_{t=0} \mb \xi(\mb a+ t \mb \delta) \right] \mb \xi(\mb a)^* \nabla \varphi_p(\mb a) \right).
\end{eqnarray}
Calculus gives 
\begin{eqnarray}
 \frac{d}{dt} \Bigr|_{t=0} \mb \xi(\mb a+ t \mb \delta)  &=& (q - 1) \frac{|\mb a|^{\circ(q-2)} \circ \mb \delta }{ \norm{ |\mb a|^{\circ(q-1)} }{2} } + \mr{sign}(\mb a) \circ |\mb a|^{\circ(q-1)} \innerprod{ \frac{1}{|\mb a|^{\circ(q-1)}} }{ \mb \delta },
\end{eqnarray}
hence 
\begin{equation}
\bb P_{\mb a^\perp}  \frac{d}{dt} \Bigr|_{t=0} \mb \xi(\mb a+ t \mb \delta)  = \bb P_{\mb a^\perp}  \left[ (q - 1) \frac{|\mb a|^{\circ(q-2)} \circ \mb \delta }{ \norm{ |\mb a|^{\circ(q-1)} }{2} } \right].
\end{equation}
At last, we can obtain the expression 
\begin{equation}
\mr{Hess}[\varphi_p] (\mb a) \mb \delta = \bb P_{\mb a^\perp} \left( \nabla^2 \varphi_p(\mb a) - (q-1) \frac{\innerprod{ \mr{sign}(\mb a) \circ |\mb a|^{\circ(q-1)}  }{ \nabla \varphi_p(\mb a) }}{ \norm{ |\mb a|^{\circ(q-1)} }{2}^2 } \mr{diag}( |\mb a|^{\circ (q-2)} ) \right) \mb \delta,
\end{equation}
for $\mb \delta \in T_{\mb a} \mc M_q$. Noting that for any such $\mb \delta$, $\bb P_{\mb a^\perp} \mb \delta = \mb \delta$, we can identify the Riemannian Hessian with the  matrix
\begin{equation}
\mr{Hess}[\varphi_p] (\mb a) \mb \delta = \bb P_{\mb a^\perp} \left( \nabla^2 \varphi_p(\mb a) - (q-1) \frac{\innerprod{ \mr{sign}(\mb a) \circ |\mb a|^{\circ(q-1)}  }{ \nabla \varphi_p(\mb a) }}{ \norm{ |\mb a|^{\circ(q-1)} }{2}^2 } \mr{diag}( |\mb a|^{\circ (q-2)} ) \right) \bb P_{\mb a^\perp}.
\end{equation}

\paragraph{Local Convexity}
Plugging in  $\bar{\mb a} = \frac{ \injector^*\shift{\mb a_0}{-\tau} }{\norm{\injector^*\shift{\mb a_0}{-\tau}}q}$ with $p=q$, we have
\begin{align}
\mathrm{Hess}[\varphi_p]\paren{\frac{ \injector^*\shift{\mb a_0}{-\tau} }{\norm{\injector^*\shift{\mb a_0}{-\tau}}q}}
&=\bb P_{\bar{\mb a}^\perp}\brac{ \nabla^2\varphi_{p}+\paren{q-1}\bar\alpha\lambda\diag\paren{\abs{\mb a}^{\circ(q-2)}} }\bb P_{\mb a^\perp}\\
& = \bb P_{\bar{\mb a}^\perp}\brac{\paren{q-1}\bar\alpha\paren{\lambda^{\frac{p-2}{p-1}}+\lambda} \diag\abs{\bar{\mb a}}^{\circ(q-2)}}\bb P_{\bar{\mb a}^\perp}.
\end{align}

Hence, the function is strictly convex at point $\bar{\mb a} = \frac{ \injector^*\shift{\mb a_0}{-\tau} }{\norm{\injector^*\shift{\mb a_0}{-\tau}}q}$ along direction $\mb \delta \in T_{\mb a} \mc M_q$ as long as whose support are contained in the support of $\mb a$. We further consider a perturbation $\mb \delta_c$ such that $\supp{\mb\delta_c}\cap\supp{\bar{\mb a}}=\emptyset$, and $\norm{\mb\delta_c}2\le\eps$, then the objective value at $\frac{\bar{\mb a}+\mb\delta_c}{ \norm{\bar{\mb a}+\mb\delta_c}q }$ satisfies
\begin{align}
\varphi_p\paren{\frac{\bar{\mb a}+\mb\delta_c}{ \norm{\bar{\mb a}+\mb\delta_c}q }} & = \frac1p\norm{\mb y-{\frac{\bar{\mb a}+\mb\delta_c}{ \norm{\bar{\mb a}+\mb\delta_c}q }}\cconv\mb x^*(\bar{\mb a}+\mb\delta_c)}p^p+\lambda\norm{\mb x^*(\bar{\mb a}+\mb\delta_c)}1\\
& \ge \frac1p\norm{\mb y-{\frac{\bar{\mb a}}{ \norm{\bar{\mb a}+\mb\delta_c}q }}\cconv\mb x^*(\bar{\mb a}+\mb\delta_c)}p^p+\lambda\norm{\mb x^*(\bar{\mb a}+\mb\delta_c)}1\\
& \ge \varphi_p\paren{\frac{\bar{\mb a}}{ \norm{\bar{\mb a}+\mb\delta_c}q }}\\
& \ge \varphi_p\paren{ \bar{\mb a} }.
\end{align}
This completes the proof.

\end{proof}

\section{Auxiliary Lemmas}
\label{sec:lemmas}
\begin{lemma}{(Lemma B.2 of \cite{SQW15-pp})}
\label{lem:mtx-sqrt-perturb}
Suppose that $\mb A \succ \mb 0$ is a positive definite matrix. For any symmetic matrix $\mb \Delta$ with $\norm{\mb \Delta}{\ell^2 \to \ell^2} \le \sigma_{\min}(\mb A) / 2$,
\begin{equation}
\norm{(\mb A + \mb \Delta)^{-1/2} - \mb A^{-1/2} }{\ell^2 \to \ell^2} \;\le\; \frac{2 \norm{\mb A}{\ell^2 \to \ell^2}^{1/2} \norm{\mb \Delta}{\ell^2 \to \ell^2} }{ \sigma_{\min}(\mb A)^2 }. 
\end{equation}
\end{lemma}

\begin{lemma} 
\label{lem:orth-perturb} 
Let $\mb U$ be a matrix such that $\mb U^* \mb U = \mb I + \mb \Delta$, with $\norm{\mb \Delta}{\ell^2 \to \ell^2} < 1/2$. Then $\mb U^* \mb U$ is invertible, 
\begin{equation}
\norm{\mb U }{\ell^2 \to \ell^2} < \sqrt{3/2},
\end{equation}
and 
\begin{equation}
\norm{ \mb U - \mb U( \mb U^* \mb U )^{-1/2} }{\ell^2 \to \ell^2} < \; 3\norm{\mb \Delta}{\ell^2 \to \ell^2}.
\end{equation}
\end{lemma}
\begin{proof}
Upper bound for the first quantity can be derived 
\begin{eqnarray}
\norm{\mb U }{\ell^2 \to \ell^2} &=& \sqrt{ \norm{\mb U^* \mb U}{\ell^2 \to \ell^2} } \\
 &\le& \sqrt{ \norm{\mb I}{\ell^2 \to \ell^2} + \norm{\mb \Delta }{\ell^2 \to \ell^2} } \\
 &<& \sqrt{3/2}. 
 \end{eqnarray}
Applying Lemma \ref{lem:mtx-sqrt-perturb} for the second term
\begin{eqnarray}
\norm{\mb U - \mb U (\mb U^* \mb U)^{-1/2} }{\ell^2 \to \ell^2} &\le& \norm{\mb U}{\ell^2\to\ell^2} \norm{\mb I - (\mb U^* \mb U)^{-1/2} }{\ell^2 \to \ell^2} \\
&=& \norm{\mb U }{\ell^2 \to \ell^2 } \norm{\mb I^{-1/2} - ( \mb I + \mb \Delta )^{-1/2} }{\ell^2 \to \ell^2 } \\
&\le& \sqrt{ 3/2 }  \times 2 \norm{\mb \Delta }{\ell^2 \to \ell^2},
\end{eqnarray}
Hence, we can obtain the claim by using $2 \sqrt{ 3/2 } < 3$ to simplify the constant.

\end{proof}

\end{document}